\documentclass[11pt]{article}   

\usepackage{graphicx}
\usepackage[round]{natbib}

\usepackage{latexsym,amsfonts,amsmath,theorem,amssymb}
\usepackage{stmaryrd,array,tabularx,bbm}
\usepackage{pstricks,graphicx}
\usepackage{booktabs}
\usepackage{a4wide}
\usepackage{arydshln}
\usepackage{color}
\usepackage{caption}
\usepackage{theorem}
\usepackage{hyperref}
\usepackage{cleveref}
\usepackage{framed}
\usepackage{blindtext}
\usepackage{float}
\usepackage{todonotes}

\usepackage{algorithm, algorithmic}

\usepackage[caption=false]{subfig}

\usepackage{tikz}

\def\ud{\mathrm d}
\newtheorem{theorem}{Theorem}[section]
\newtheorem{lemma}[theorem]{Lemma}
\newtheorem{proposition}[theorem]{Proposition}
\newtheorem{assumption}[theorem]{Assumption}

\newenvironment{proof}{\begin{trivlist}
    \item[\hskip\labelsep{\bf Proof.}]}{$\hfill\Box$\end{trivlist}}
\newenvironment{proofof}[1]{\begin{trivlist}
    \item[\hskip\labelsep{\bf Proof of {#1}.}]}{$\hfill\Box$\end{trivlist}}
{\theoremstyle{plain} \theorembodyfont{\rmfamily}
\newtheorem{example}[theorem]{Example}
}

\numberwithin{equation}{section}
\numberwithin{figure}{section}
\numberwithin{table}{section}

\newcommand{\e}{{\varepsilon}}
\newcommand{\norm}[1]{\left\|#1 \right\|}  
\newcommand{\abs}[1]{\left|#1\right|}  
\newcommand{\ip}[1]{\left\langle#1\right\rangle}

\newcommand{\cW}{{{\mathcal W}}}
\newcommand{\E}{{{\mathbb E}}}

      \newcommand{\TV}{{{\rm TV}}}
    \newcommand{\R}{{{\mathbb{R}}}}
    \newcommand{\mP}{{{\mathbb{P}}}}

\renewcommand{\hat}{\widehat}
\newcommand{\vertiii}[1]{{\left\vert\kern-0.25ex\left\vert\kern-0.25ex\left\vert #1 
\right\vert\kern-0.25ex\right\vert\kern-0.25ex\right\vert}}

\title{How to beat a Bayesian adversary}
\author{Zihan Ding\footnote{Department of Electrical and Computer Engineering, Princeton University, USA, \texttt{zihand@princeton.edu}} \and Kexin Jin\footnote{Department of Mathematics, Princeton University,  USA, \texttt{kexinj@math.princeton.edu}} \and Jonas Latz\footnote{Department of Mathematics, University of Manchester, UK, \texttt{jonas.latz@manchester.ac.uk}} \and Chenguang Liu\footnote{Delft Institute of Applied Mathematics, Technische Universiteit Delft, The Netherlands, \texttt{c.liu-13@tudelft.nl}} }

\begin{document}

\maketitle
\begin{abstract}  Deep neural networks and other modern machine learning models are often susceptible to  adversarial attacks. Indeed, an adversary may often be able to change a model's prediction through a small, directed perturbation of the model's input -- an issue in safety-critical applications. Adversarially robust machine learning is usually based on a minmax optimisation problem that minimises the machine learning loss under maximisation-based adversarial attacks.

In this work, we study adversaries that determine their attack using a Bayesian statistical approach rather than maximisation. The resulting Bayesian adversarial robustness problem is a relaxation of the usual minmax problem. To solve this problem, we propose Abram -- a continuous-time particle system that shall approximate the gradient flow corresponding to the underlying learning problem. We show that Abram approximates a McKean-Vlasov process and justify the use of Abram by giving assumptions under which the McKean-Vlasov process finds the minimiser of the Bayesian adversarial robustness problem. We discuss two ways to discretise Abram and show its suitability  in benchmark adversarial deep learning experiments.
\end{abstract}



\section{Introduction} \label{Sec_intro}
Machine learning and artificial intelligence play a major role in today's society: self-driving cars (e.g., \cite{Bachute}), automated medical diagnoses (e.g., \cite{Rajkomar}), and security systems based on face recognition (e.g., \cite{Sharma}), for instance, are often based on certain machine learning models, such as \textit{deep neural networks} (DNNs). DNNs are often discontinuous with respect to their input \citep{szegedy2014intriguing} making them susceptible to so-called \emph{adversarial attacks}. In an adversarial attack, an adversary aims to change the prediction of a DNN through a directed, but small perturbation to the input. We refer to \cite{goodfellow2015explaining} for an example showing the weakness of DNNs towards adversarial attacks. Especially when employing DNNs in safety-critical applications, the training of machine learning models in a way that is robust to adversarial attacks has become a vital task. 

Machine learning models are usually trained by minimising an associated loss function. In adversarially robust learning, this loss function is considered to be subject to adversarial attacks. The adversarial attack is usually given by a perturbation of the input data that is chosen to maximise the loss function. Thus, adversarial robust learning is formulated as a minmax optimisation problem. 
In practice, the inner maximisation problem needs to be approximated: \cite{goodfellow2015explaining} proposed the Fast Gradient Sign Method (FGSM) which perturbs the input data to maximise the loss function with a single step. Improvements of FGSM were proposed by, e.g., \cite{tra2018ensemble, kurakin2017adversarial, Wong2020Fast}. Another popular methodology is Projected Gradient Descent (PGD)  \citep{madry2018towards} and its variants, see, for example, \cite{8579055, pmlr-v97-yang19e, mosbach2019logit, NEURIPS2019_5d4ae76f, pmlr-v119-maini20a, pmlr-v119-croce20b}. Similar to FGSM, PGD considers the minmax optimisation problem  but uses multi-step gradient ascent to approximate the inner maximisation problem. Notably, \cite{Wong2020Fast} showed that FGSM with random initialization is as effective as PGD.

Other defense methods include preprocessing (e.g. \citealt{XU1, song2018pixeldefend, guo2018countering, 9008296}) and detection (e.g. \citealt{10.1145/3128572.3140444, metzen2017on, 9878604, 10030535}), as well as provable defenses (e.g. \citealt{pmlr-v80-wong18a, 9010971, sheikholeslami2021provably, jia2022multiguard}). Various attack methods have also been proposed, see, for instance, \cite{7958570, pmlr-v97-wong19a, Ghiasi2020BREAKING, pmlr-v119-croce20b}. More recently, there is an increased focus on using generative models to improve adversarial accuracy, see for example \cite{pmlr-v162-nie22a, pmlr-v202-wang23ad, NEURIPS2023_088463cd}.

In the present work, we study the case of an adversary that finds their attack following a Bayesian statistical methodology. The Bayesian adversary does not find the attack through optimisation, but by sampling a probability distribution that can be derived using Bayes' Theorem. Importantly, we study the setting in which the adversary uses a Bayesian strategy, but the machine learner/defender trains the model using optimisation, which is in contrast to \cite{Ye2018}. The associated Bayesian adversarial robustness problem can be interpreted as a stochastic relaxation of the classical minmax problem that replaces the inner maximisation problem with an integral.  After establishing this connection, we 
\begin{itemize}
    \item propose \emph{Abram} (short for \emph{Adversarial Bayesian Particle Sampler}), a particle-based continuous-time dynamical system that simultaneously approximates the behaviour of the Bayesian adversary and trains the model via gradient descent.
    \end{itemize}
    Particle systems of this form have been used previously to solve such optimisation problems in the context of maximum marginal likelihood estimation, see, e.g., \cite{Akyildiz} and \cite{Kuntz}. In order to justify the use Abram in this situation, we
    \begin{itemize}
    \item show that Abram converges to a McKean--Vlasov stochastic differential equation as the number of particles goes to infinity, and 
    \item give assumptions under which the McKean-Vlasov SDE converges to the minimiser of the Bayesian adversarial robustness problem with an exponential rate.
    \end{itemize}
Additional complexity arises here compared to earlier work as the dynamical system and its limiting McKean-Vlasov SDE have to be considered under reflecting boundary conditions. After the analysis of the continuous-time system, we briefly explain its discretisation. Then, we
    \begin{itemize}
    \item compare Abram to the state-of-the-art in adversarially robust classificaton of the MNIST and the CIFAR-10 datasets under various kinds of attacks.
\end{itemize}

This work is organised as follows. We introduce the (Bayesian) adversarial robustness problem in Section~\ref{Subs:BayeAdvPartS} and the Abram method in Section~\ref{Sec_Abram}. We analyse Abram in Sections~\ref{Sec_propChaos} (large particle limit) and \ref{Sec_McVl} (longtime behaviour). We  discuss different ways of employing Abram in practice in Section~\ref{Sec_Discre} and compare it to the state-of-the-art in adversarially robust learning in Section~\ref{Sec_Exp}. We conclude in Section~\ref{Sec_concl}.

\section{Adversarial robustness and its Bayesian relaxation} \label{Subs:BayeAdvPartS}
In the following, we consider a supervised machine learning problem of the following form. We are given a \emph{training dataset} $\{(y_1, z_1),\ldots,(y_K, z_K)\}$ of pairs of features $y_1,\ldots,y_K \in Y :=\mathbb{R}^{d_Y}$ and \emph{labels} $z_1,\ldots,z_K \in Z$. Moreover, we are given a parametric model of the form $g: X \times Y \rightarrow Z$, with $X := \mathbb{R}^d$ denoting the parameter space. Goal is now to find a parameter $\theta^*$, for which 
$$
g(y_k|\theta^*) \approx z_k \qquad (k=1,\ldots,K).
$$
In practice the function $g(\cdot|\theta^*)$ shall then be used to predict labels of features (especially such outside of training dataset).

The parameter $\theta^*$ is usually found through optimisation. Let $\ell:Z \times Z \rightarrow \mathbb{R}$ denote a loss function -- a function that gives a reasonable way of comparing the output of $g$ with observed labels. Then, we need to solve the following optimisation problem:
\begin{align} \label{eq:learn_non_adv}
    \min_{\theta \in X} & \frac{1}{K}\sum_{k=1}^K \Phi(y_k,z_k|\theta)
\end{align}
where $\Phi(y,z|\theta) := \ell(g(y|\theta),z)$.

Machine learning models $g$ that are trained in this form are often susceptible to adversarial attacks. That means, for a given feature vector $y$, we can find a `small' $\xi \in Y$ for which $g(y+\xi|\theta^*) \neq g(y|\theta^*)$. In this case, an adversary can change the model's predicted label by a very slight alteration of the input feature.
Such a $\xi$ can usually be found through optimisation on the input domain: 
$$\max_{\xi \in B(\varepsilon)}\Phi(y+\xi, z|\theta),$$
where $B(\varepsilon) = \{\xi: \|\xi\|\leq \varepsilon\}$ denotes the $\varepsilon$-ball centred at $0$ and $\varepsilon > 0$ denotes the size of the adversarial attack. Hence, the attacker tries to change the prediction of the model whilst altering the model input only by a small value $\leq \varepsilon$. Other kinds of attacks are possible, the attacker may, e.g., try to not only change the predicted label to a \emph{some} other label, but actually change it to a particular other label, see, e.g., \cite{kurakin2017adversarial2}.

In adversarially robust training, we replace the optimisation problem \eqref{eq:learn_non_adv} by the minmax optimisation problem below:
\begin{equation} \label{eq:learn_adv}
    \min_{\theta \in X} \frac{1}{K}\sum_{k=1}^K \max_{\xi_k \in B(\varepsilon)}\Phi(y_k+\xi_k,z_k|\theta).
\end{equation}
Thus, we now train the network by minimising the loss also with respect to potential adversarial attacks.
Finding the accurate solutions to such minmax optimisation problems is difficult: there is no underlying saddlepoint structure, e.g., $\Phi(y,z|\theta)$ is neither convex in $\theta$ nor concave in $y$, $X$ and $Y$ tend to be very high-dimensional spaces, and the number of datasets $K$ may prevent the accurate computation of gradients. However, good heuristics have been established throughout the last decade -- we have mentioned some of them in Section~\ref{Sec_intro}.

In this work, we aim to study a relaxed version of the minmax problem, which we refer to as the \emph{Bayesian adversarial robustness problem}. This problem is given by
\begin{equation} \label{eq_BayesianAdv}
    \min_{\theta \in X} \frac{1}{K}\sum_{k =1}^K \int \Phi(y_k+\xi_k,z_k|\theta) \pi^{\gamma, \varepsilon}_k(\mathrm{d}\xi_k|\theta),
\end{equation}
where the \emph{Bayesian adversarial distribution} $\pi^{\gamma, \varepsilon}_k(\cdot|\theta)$ has (Lebesgue) density 
$$
\xi \mapsto \frac{\exp(\gamma\Phi(y_k+\xi,z_k|\theta)) \mathbf{1}[\xi \in B(\varepsilon)]}{\int_{B(\varepsilon)} \exp(\gamma\Phi(y_k+\xi',z_k|\theta)) \mathrm{d}\xi'},
$$
where $\gamma > 0$ is an \emph{inverse temperature}, $\varepsilon > 0$ still denotes the size of the adversarial attack, and $\mathbf{1}[\cdot]$ denotes the indicator: $\mathbf{1}[\mathrm{true}] := 1$ and $\mathbf{1}[\mathrm{false}] := 0$. The distribution $\pi^{\gamma, \varepsilon}_k(\cdot|\theta)$ is concentrated on the $\varepsilon$-ball, $\varepsilon>0$ controls the range of the attack, $\gamma >0$ controls its focus.  We illustrate this behaviour in Figure~\ref{fig:epsvsgamma}.
Next, we comment on the mentioned relaxation and the Bayesian derivation  of this optimisation problem.

\begin{figure}
    \centering
    \includegraphics[scale = 0.95]{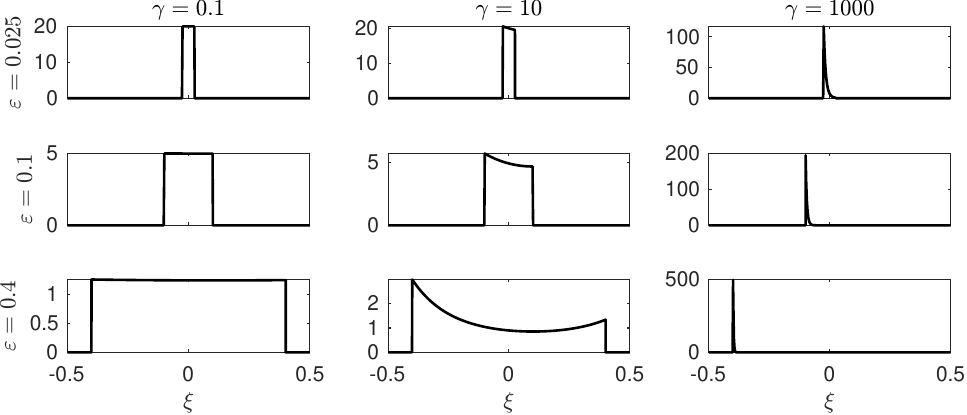}
    \caption{Plots of the Lebesgue density of $\pi_1^{\gamma, \varepsilon}(\cdot|\theta_0)$ for energy $\Phi(y_1 + \xi, z_1|\theta_0) = (\xi-0.1)^2/2$, choosing parameters $\varepsilon \in \{0.025, 0.1, 0.4\}$ and $\gamma \in \{0.1, 10, 1000\}$.}
    \label{fig:epsvsgamma}
\end{figure}

\paragraph{Relaxation.}
Under certain assumptions on $\Phi$, one can show that $$\pi^{\gamma, \varepsilon}_k(\cdot|\theta) \rightarrow \mathrm{Unif}(\mathrm{argmax}_{\xi\in Y} \Phi(y_k + \xi,z_k|\theta))$$ weakly as $\gamma \rightarrow \infty$, see \cite{Hwang}. Indeed, the Bayesian adversarial distribution converges to the uniform distribution over the global maximisers computed with respect to the adversarial attack. This limiting behaviour, that we can also see in Figure~\ref{fig:epsvsgamma}, forms the basis of simulated annealing methods for global optimisation. Moreover, it implies that the optimisation problems \eqref{eq:learn_adv} and \eqref{eq_BayesianAdv} for $\gamma = 0$ are identical, since
\begin{equation*}
    \frac{1}{K}\sum_{k =1}^K \int \Phi(y_k+\xi_k,z_k|\theta) \pi^{0, \varepsilon}_k(\mathrm{d}\xi_i|\theta) =  \frac{1}{K}\sum_{k =1}^K \int \Phi(y_k+\xi_k,z_k|\theta) \mathrm{Unif}(\mathrm{argmax}_{\xi\in Y} \Phi(y_k + \xi,z_k|\theta))(\mathrm{d}\xi_i)
\end{equation*}
and since $\xi_k \sim \mathrm{Unif}(\mathrm{argmax}_{\xi\in Y} \Phi(y_k + \xi,z_k|\theta))$ implies 
$\Phi(y_k+\xi_k,z_k|\theta) = \max_{\xi \in B(\varepsilon)}\Phi(y_k+\xi,z_k|\theta)$ for $k =1,\ldots,K$.
A strictly positive $\gamma$ on the other hand leads to a relaxed problem circumventing the minmax optimisation. 
\cite{cipriani2024minimaxoptimalcontrolapproach} have also discussed this relaxation of an adversarial robustness problem in the context of a finite set of attacks, i.e. the $\varepsilon$-ball $B(\varepsilon)$ is replaced by a finite set.

\paragraph{Bayesian.} We can understand the kind of attack that is implicitly employed in \eqref{eq_BayesianAdv} as a Bayesian attack. We now briefly introduce the Bayesian learning problem to then explain its relation to this adversarial attack. In Bayesian learning, we model $\theta$ as a random variable with a so-called \emph{prior (distribution)} $\pi_{\rm prior}$. The prior incorporates information about $\theta$. In Bayesian learning, we now inform the prior about data $\{(y_1, z_1),\ldots,(y_K, z_K)\}$  by conditioning $\theta$ on that data. Indeed, we train the model by finding the conditional distribution of $\theta$ given that $g(z_k|\theta) \approx y_k$ ($k =1,\ldots,K)$. In the Bayesian setting, we represent `$\approx$' by a noise assumption consistent with the loss function $\ell$. This is achieved by defining the so-called \emph{likelihood} as $\exp(-\Phi)$. The conditional distribution describing $\theta$ is called \emph{posterior (distribution)} $\pi_{\rm post}$ and can be obtained through Bayes' theorem, which states that 
$$
\pi_{\rm post}(A) = \frac{\int_A \exp\left(-\frac{1}{K}\sum_{k=1}^K\Phi(y_k, z_k|\theta)\right) \pi_{\rm prior}(\mathrm{d}\theta)}{\int_X \exp\left(-\frac{1}{K}\sum_{i=k}^K\Phi(y_k, z_k|\theta')\right) \pi_{\rm prior}(\mathrm{d}\theta')}
$$
for measurable $A \subseteq X$.
A model prediction with respect to feature $z$ can then be given by the posterior mean of the output $g$, which is 
$$
\int g(z|\theta) \pi_{\rm post}(\mathrm{d}\theta)
$$
The Bayesian attacker treats the attack $\xi_k$ in exactly such a Bayesian way. They define a prior distribution for the attack, which is the uniform distribution over the $\varepsilon$-ball: $$\mathrm{Unif}(B(\varepsilon)) = \int_{B(\varepsilon)}\mathbf{1}[\xi_k \in \cdot ]\mathrm{d}\xi_k.$$ The adversarial likelihood is designed to essentially cancel  out the likelihood in the Bayesian learning problem, by defining a function that gives small mass to the learnt prediction and large mass to anything that does not agree with the learnt prediction:
$$
\exp(\gamma\Phi(y_k+\xi_k,z_k|\theta)).
$$
Whilst this is not a usual likelihood corresponding to a particular noise model, we could see this as a special case of \emph{Bayesian forgetting} \citep{FuHeXu}. In Bayesian forgetting, we would try to remove a single dataset from a posterior distribution by altering the distribution of the parameter $\theta$. In this case, we try to alter the knowledge we could have gained about the feature vector by altering that feature vector to produce a different prediction.

\section{Adversarial Bayesian Particle Sampler} \label{Sec_Abram}
We now derive a particle-based method that shall solve \eqref{eq_BayesianAdv}. To simplify the presentation in the following, we  assume that $K=1$, i.e., there is only a single data set. The derivation for multiple data sets is equivalent -- computational implications given by multiple datasets will be discussed in Section~\ref{Sec_Discre}. We also ignore the dependence of $\Phi$ on particular datasets and note only the dependence on parameter and attack. Indeed, we write \eqref{eq_BayesianAdv} now as
$$
\min_{\theta \in X} F(\theta) := \int \Phi(\xi,\theta) \pi^{\gamma, \varepsilon}(\mathrm{d}\xi|\theta).
$$
To solve this minimisation problem, we study the gradient flow corresponding to the energy $F$, that is: ${\mathrm{d}\zeta_t} = -\nabla_\zeta F(\zeta_t){\mathrm{d}t}$. The gradient flow is a continuous-time variant of the gradient descent algorithm. The gradient flow can be shown to converge to a minimiser of $F$ in the longterm limit if $F$ satisfies certain regularity assumptions. The gradient of $F$ has  a rather simple expression:
\begin{align*}
    \nabla_\theta F(\theta) 
    &= \nabla_\theta\frac{ \int_{B(\varepsilon)}\Phi(\xi,\theta)\exp(\gamma\Phi(\xi',\theta)) \mathrm{d}\xi'}{\int_{B(\varepsilon)} \exp(\gamma\Phi(\xi',\theta)) \mathrm{d}\xi'} \\
    &= \frac{ \int_{B(\varepsilon)} \nabla_\theta\Phi(\xi,\theta) \cdot \exp(\gamma\Phi(\xi',\theta)) + \gamma\nabla_\theta\Phi(\xi,\theta) \cdot \Phi(\xi,\theta)\exp(\gamma\Phi(\xi',\theta)) \mathrm{d}\xi'}{\int_{B(\varepsilon)} \exp(\gamma\Phi(\xi',\theta)) \mathrm{d}\xi'} \\
    &\qquad - \frac{\left( \int_{B(\varepsilon)}\Phi(\xi,\theta)\exp(\gamma\Phi(\xi',\theta)) \mathrm{d}\xi'\right)\left(\int_{\|\cdot\| \leq \varepsilon } \gamma \nabla_\theta\Phi(\xi,\theta) \cdot\exp(\gamma\Phi(\xi',\theta)) \mathrm{d}\xi'\right)}{\left(\int_{B(\varepsilon)}\exp(\gamma\Phi(\xi',\theta)) \mathrm{d}\xi'\right)^2} \\
    &= \int \nabla_\theta\Phi(\xi,\theta) \pi^{\gamma, \varepsilon}(\mathrm{d}\xi|\theta) + \gamma\mathrm{Cov}_{\pi^{\gamma, \varepsilon}(\cdot|\theta)}(\Phi(\cdot,\theta), \nabla_\theta \Phi(\cdot,\theta)),
\end{align*}
where we assume that $\Phi$ is continuously differentiable, bounded below, and sufficient regularity to be allowed here to switch gradients and integrals. As usual, we define the covariance of appropriate functions $f,g$ with respect to a probability distribution $\pi$, by
$$
\mathrm{Cov}_{\pi}(f,g) := \int_X f(\theta)g(\theta) \pi(\mathrm{d}\theta) - \int_X f(\theta') \pi(\mathrm{d}\theta')\int_X g(\theta'') \pi(\mathrm{d}\theta''). 
$$
The structure of $\nabla_\theta F$ is surprisingly simple, requiring only integrals of the target function and its gradient with respect to $\pi^{\gamma, \varepsilon}$, but, e.g., not its normalising constant. In practice, it is usually not possible to compute these integrals analytically or to even sample independently from $\pi^{\gamma, \varepsilon}(\cdot|\theta)$, which would be necessary for a stochastic gradient descent approach. The latter approach first introduced by \cite{RobbinsMonro} allows the minimisation of expected values by replacing these expected values by sample means; see also \cite{Jin1} and \cite{Latz} for continuous-time variants. Instead, we use a particle system approach that has been studied for a different problem by \cite{Akyildiz} and \cite{Kuntz}. The underlying idea is to approximate $\pi^{\gamma, \varepsilon}(\cdot|\theta)$ by an overdamped Langevin dynamics which is restricted to the $\varepsilon$-Ball $B(\varepsilon)$ with reflecting boundary conditions:
$$\mathrm{d}\xi_t = \gamma \nabla_\xi \Phi(\xi_t,\theta)\mathrm{d}t + \sqrt{2}\mathrm{d}W_t,$$
where $(W_t)_{t \geq 0}$ denotes a standard Brownian motion on $Y$.
Under weak assumptions on $\Phi$, this Langevin dynamics  converges to the distribution $\pi^{\gamma, \varepsilon}(\cdot|\theta)$ as $t \rightarrow \infty$. However, in practice, we are not able to simulate the longterm behaviour of this dynamics for all fixed $\theta$ to produce samples of $\pi^{\gamma, \varepsilon}(\cdot|\theta)$ as required for stochastic gradient descent. Instead, we run a number $N$ of (seemingly independent) Langevin dynamics $(\xi_t^{1,N})_{t \geq 0}, \ldots, (\xi_t^{N,N})_{t \geq 0}$. We then obtain an approximate gradient flow $(\theta_t^N)_{t \geq 0}$ that uses the ensemble of particles $(\xi_t^{1,N})_{t \geq 0}, \ldots, (\xi_t^{N,N})_{t \geq 0}$ to approximate the expected values in the gradient $\nabla_\theta F$ and then feed $(\theta_t^N)_{t \geq 0}$ back into the drift of the $(\xi_t^{1,N})_{t \geq 0}, \ldots, (\xi_t^{N,N})_{t \geq 0}$. Hence, we simultaneously approximate the gradient flow $(\zeta_t)_{t \geq 0}$ by $(\theta_t)_{t \geq 0}$ and the Bayesian adversarial distribution $(\pi^{\gamma, \varepsilon}(\cdot|\theta_t))_{t \geq 0}$ by $(\xi_t^{1,N})_{t \geq 0}, \ldots, (\xi_t^{N,N})_{t \geq 0}$. Overall, we obtain the dynamical system
\begin{align*}
    \mathrm{d}\theta_t^N &= - \frac{1}{N}\sum_{n=1}^N\nabla_\theta \Phi(\xi_t^{n,N},\theta_t^N)\mathrm{d}t -  \gamma\widehat{\mathrm{Cov}}(\xi_t^N)\mathrm{d}t \\
    \mathrm{d}\xi_t^{i,N} &= \gamma \nabla_\xi  \Phi(\xi_t^{i,N},\theta_t)\mathrm{d}t + \sqrt{2}\mathrm{d}W_t^{i} \qquad (i  =1,...,N).
\end{align*}
where $(W_t^{i})_{t \geq 0}$ are mutually independent Brownian motions on $Y$ for $i = 1,...,N$. Again, the Langevin dynamics $(\xi_t^{1,N})_{t \geq 0}, \ldots, (\xi_t^{N,N})_{t \geq 0}$ are defined on the ball $B(\varepsilon)$ with reflecting boundary conditions -- we  formalise this fact below. The empirical covariance is given by
$$\widehat{\mathrm{Cov}}(\xi_t^N) = \frac1{N} \sum_{i=1}^N \Phi(\xi_t^{i,N},\theta_t)\nabla_\theta \Phi(\xi_t^{i,N},\theta_t) - \frac1{N^2}\sum_{i'=1}^K \Phi(\xi_t^{i',N},\theta_t) \sum_{i''=1}^K\nabla_\theta \Phi(\xi_t^{i'',N},\theta_t).$$ 
We refer to the dynamical system $(\theta^N_t, \xi^{1,N}_t,\ldots,\xi^{N,N}_t)_{t \geq 0}$ as \emph{Abram}.
We illustrate the dynamics of Abram in Figure~\ref{fig:paths}, where we consider a simple example.

\begin{figure}
    \centering
    \includegraphics[scale=0.70]{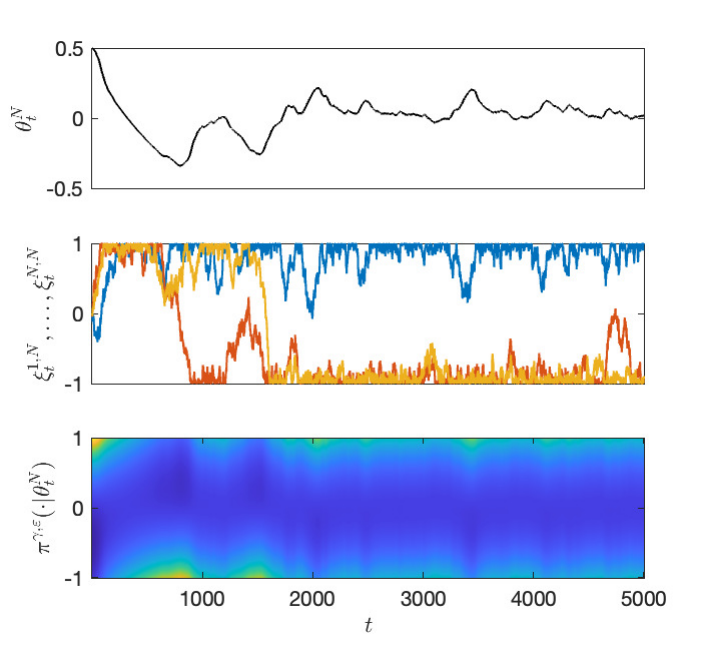}
    \includegraphics[scale=0.70]{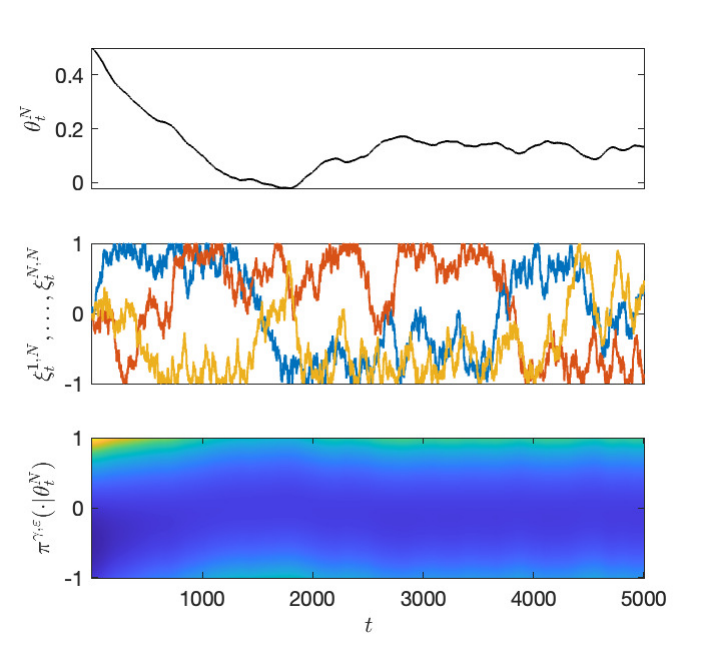}

    \includegraphics[scale=0.70]{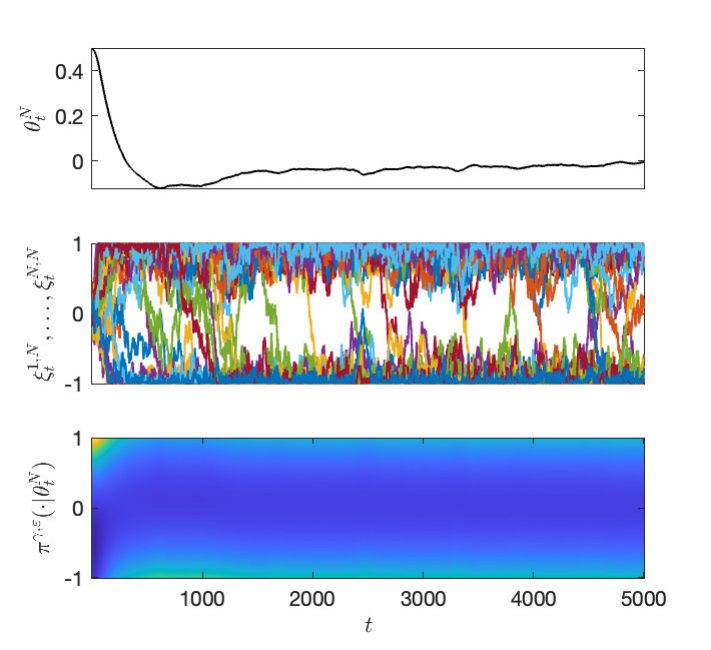}
    \includegraphics[scale=0.70]{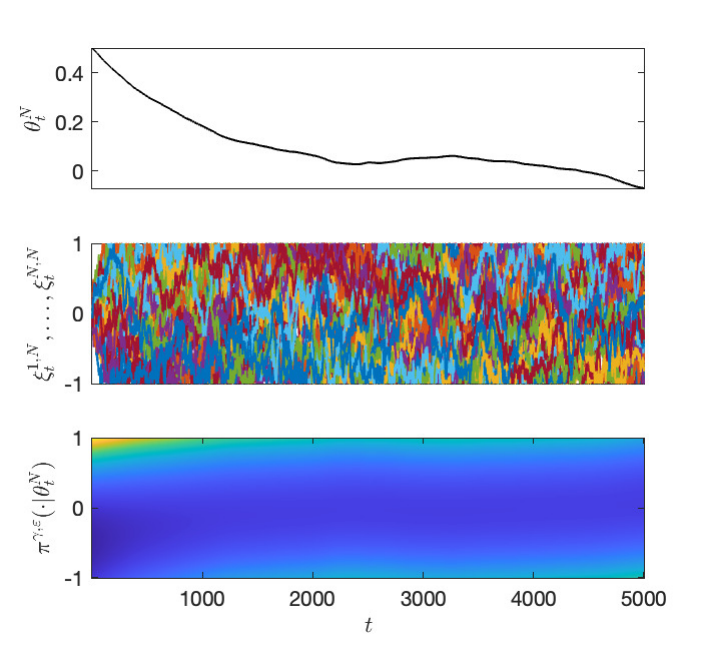}
    \caption{Examples of the Abram method given $\Phi(\xi, \theta) = \frac{1}{2}(\xi + \theta)^2$, $\varepsilon = 1$, and different combinations of $(\gamma, N) = (10,3)$ (top left), $(0.1,3)$ (top right), $(10,50)$ (bottom left), $(0.1, 50)$ (bottom right). In each of the four quadrants, we show the simulated path $(\theta_t^N)_{t \geq 0}$ (top), the particle paths $(\xi_t^{1,N},\ldots,\xi_t^{N,N})_{t \geq 0}$ (centre), and the path of probability distributions $(\pi^{\gamma, \varepsilon}(\cdot|\theta_t^N))_{t \geq 0}$ (bottom) that shall be approximated by the particles. The larger $\gamma$ leads to a concentration of $\pi^{\gamma, \varepsilon}$ at the boundary, whilst it is closer to uniform if $\gamma$ is small. More particles lead to a more stable path $(\theta^N_t)_{t \geq 0}$. A combination of large $N$ and $\gamma$ leads to convergence to the minimiser $\theta_* = 0$ of $F$. }
    \label{fig:paths}
\end{figure}

We have motivated this particle system as an approximation to the underlying gradient flow $(\zeta_t)_{t \geq 0}$. As $N \rightarrow \infty$, the dynamics $(\theta^N_t)_{t \geq 0}$ does not necessarily convergence to the gradient flow $(\zeta_t)_{t \geq 0}$, but to a certain McKean--Vlasov Stochastic Differential Equation (SDE), see \cite{McKean}. We study this convergence behaviour in the following, as well as the  convergence of the McKean--Vlasov SDE to the minimiser of $F$ and, thus, justify Abram as a method for Bayesian adversarial learning. First, we introduce the complete mathematical set-up and give required assumptions. 

\subsection{Mean-field limit}
In the following, we are interested in the mean field limit of Abram, i.e., we analyse the limit of $(\theta^N_t)_{t  \geq 0}$ as $N \rightarrow \infty$. Thus, we can certainly assume for now that $\gamma := 1$ and $\varepsilon \in (0,1)$ being fixed. We write $B := B(\varepsilon)$. Then, Abram $(\theta^N_t, \xi^{1,N}_t,\ldots,\xi^{N,N}_t)_{t \geq 0}$ satisfies 
\begin{align}\label{main}
    \theta^N_t &= \theta_0 - \int_0^t\mu_{c}^N(\nabla_\theta \Phi(\cdot,\theta^N_s))\mathrm{d}s -  \int_0^t \mathrm{Cov}_{\mu_{c}^N}(\Phi(\cdot,\theta^N_s),\nabla_\theta \Phi(\cdot,\theta^N_s))\mathrm{d}s\nonumber \\
    \xi_t^{i,N} &=\xi^i_0 +\int_0^t\nabla_x \Phi(\xi_s^{i,N},\theta^N_s)\mathrm{d}s + \sqrt{2}W^i_t+\int_0^tn(\xi_s^{i,N})\mathrm{d}l^{i,N}_s\qquad  (i=1,..., N).
\end{align}
Here, $(W_t^1)_{t \geq 0},\ldots, (W_t^N)_{t \geq 0}$ are independent Brownian motions on $Y$ and the initial particle values $\xi^1_0,\ldots,\xi^N_0$ are independent and identically distributed. There and throughout the rest of this work, we denote the expectation of some appropriate function $f$ with respect to a probability measure $\pi$ by $\pi(f) := \int_X f(\theta) \pi(\mathrm{d}\theta).$ We use $\mu_t^N$ to denote the empirical distribution of the particles $(\xi_t^{1,N}, \ldots, \xi_t^{N,N})$ at time $t \geq 0$. That is $\mu_t^N:=\frac{1}{N}\sum_{i=1}^N\delta(\cdot - {\xi_t^{i, N}})$, where $\delta(\cdot - \xi)$ is the Dirac mass concentrated in $\xi \in B$. This implies especially that we can write
\begin{align*}
\mu_t^N(f) = \frac{1}{N}\sum_{i=1}^N  f(\xi_t^{i,N}), \qquad 
  \mathrm{Cov}_{\mu_t^N}(f,g) = \frac{1}{N}\sum_{i=1}^N f(\xi_t^{i,N}g(\xi_t^{i,N})-\frac{1}{N^2}\sum_{i=1}^N \sum_{j=1}^N f(\xi_t^{i,N})g(\xi_t^{j,N}),
\end{align*}
for appropriate functions $f$ and $g$.
The particles are constrained in $B$ by the last term in the equations of the $(\xi_t^{1,N}, \ldots, \xi_t^{N,N})_{t \geq 0}$. Here, $n(x)=-x/\norm{x}$ for $x\in \partial B$ is the inner normal vector field and
$l^{i,N}$'s are a non-decreasing functions  with $l^{i,N}(0)=0$ and $\int_0^t \boldsymbol{1}[{\xi^{i,N}_s\notin\partial B(\varepsilon)}]dl^{i,N}(s)=0$, see \cite{Pilipenko}  for details on reflecting boundary conditions in diffusion processes. 
Additionally, it is convenient to define
\begin{align*}
    G(\theta,\nu)= \nabla_\theta\Big[\nu( \Phi(\cdot,\theta))+ \mathrm{Var}_\nu[\Phi(\cdot,\theta)]/2\Big]=\nu(\nabla_\theta \Phi(\cdot,\theta))+\mathrm{Cov}_{\nu}(\Phi(\cdot ,\theta),\nabla_\theta \Phi(\cdot,\theta)).
\end{align*}
for any probability measure $\nu$ on ${B}$ and $\theta\in X$.

 We finish this background section by defining the limiting McKean--Vlasov SDE with reflection
\begin{align}\label{eq: limit}
      \theta_t &= \theta_0- \int_0^t\mu_s(\nabla_\theta \Phi(\cdot,\theta_s))\mathrm{d}s -  \int_0^t \mathrm{Cov}_{\mu_s}(\Phi(\cdot ,\theta_s),\nabla_\theta \Phi(\cdot ,\theta_s))\mathrm{d}s,\nonumber \\
    \xi_t &=\xi_0+ \int_0^t\nabla_x \Phi(\xi_s,\theta_s)\mathrm{d}s + \sqrt{2}W_t+\int_0^tn(\xi_s)\mathrm{d}l_s,
\end{align}
 with $\mu_t$ denoting the law of $\xi_t$ at time $t \geq 0$. Goal of this work is to show that the particle system \eqref{main} converges to this McKean-Vlasov SDEs as $N\to\infty$ and to then show that the McKean-Vlasov SDE can find the minimiser of $F$.

\subsection{Assumptions} \label{Subsc: Assump}
We now list assumptions that we consider throughout this work. We start with the Lipschitz continuity of $\nabla \Phi$ and $G$.

\begin{assumption}[Lipschitz]\label{ass: lip}
The function $\nabla_\xi \Phi$
is Lipschitz continuous, i.e. there exists a Lipschitz constant $L>1$ such that 
\begin{align*}
    \norm{\nabla_x \Phi(x,\tilde\theta)-\nabla_x \Phi(\tilde x,\tilde\theta)}\le L\Big(\norm{x-\tilde x}+\norm{\theta-\tilde\theta}\Big)
\end{align*}
for any $\xi, \tilde \xi \in B$ and $\theta,\tilde\theta\in \R^n$.
Similarly, we assume that $G(\theta, \mu)$ is Lipschitz in the following sense: there is an $L>1$ such that 
\begin{align*}
    \norm{G(\theta,\nu)-G(\tilde\theta,\tilde\nu)}\le L\Big(\norm{\theta-\tilde\theta}+\cW_1(\nu,\tilde \nu)\Big)
\end{align*}
for any probability measures $\nu$, $\tilde\nu$ on $B$ and $\theta,\ \tilde \theta\in \R^n.$

\end{assumption}
 In Assumption~\ref{ass: lip} and throughout this work, $\cW_p$ denotes the \emph{Wasserstein-$p$ distance} given by
\begin{align*}
    \cW^p_p(\nu,\nu') =  \inf\left\lbrace\int_{X\times X}\norm{y-y'}^p\Gamma(\mathrm{d}y,\mathrm{d}y') : \Gamma \text{ is a coupling of } \nu, \nu'\right\rbrace,
\end{align*}
for probability distributions $\nu, \nu'$ on $(X, \mathcal{B}X)$ and $p \geq 1$. In addition to the Wasserstein distance, we sometimes measure the distance between probability distributions $\nu, \nu'$ on $(X, \mathcal{B}X)$ using the \emph{total variation distance} given by
$$
\|\nu -\nu'\|_{\TV} = \sup_{A \in \mathcal{B}X}|\nu(A)-\nu'(A)|.
$$
The Lipschitz continuity of $G$ actually already implies the Lipschitz continuity of $\nabla_\theta\Phi$. By setting $\nu= \delta(\cdot -x)$ and $\tilde \nu=\delta(\cdot - {\tilde x})$, we have
\begin{align*}
    \norm{\nabla_\theta \Phi(x,\tilde\theta)-\nabla_\theta \Phi(\tilde x,\tilde\theta)} &= \norm{G(\theta,\delta(\cdot -x))-G(\tilde\theta,\delta(\cdot -\tilde x))} \\ &\le L\Big(\norm{\theta-\tilde\theta}+\cW_1(\delta(\cdot - x),\delta(\cdot - \tilde x))\Big) = L\Big(\norm{x-\tilde x}+\norm{\theta-\tilde\theta}\Big) .
\end{align*} 
We assume throughout that the constant $L>1$ to simplify the constants in the Theorem \ref{th: limit}.
Finally, we note that Assumption~\ref{ass: lip} implies the well-posedness of  both \eqref{main}   and \eqref{eq: limit}, see \citet[Theorems 3.1, 3.2]{Adams}.

Next, we assume the strong convexity of $G$, which, as we note below, also implies the strong convexity of $\Phi(x,\cdot)$ for any $x \in B$. This assumption is not realistic in the context of deep learning (e.g., \cite{pmlr-v38-choromanska15}), but not unusual when analysing learning techniques.
\begin{assumption}[Strong convexity]\label{ass: convex}
   For any probability measure $\nu$ on $(B, \mathcal{B}B)$, $G(\cdot ,\nu)$ is $2\lambda$-strongly convex, 
i.e., for any $\theta, \tilde\theta \in \R^n$, we have
\begin{align*}
    \ip{G(\theta,\nu)-G(\tilde\theta,\nu), \theta-\tilde\theta}\ge 2\lambda\norm{\theta-\tilde\theta}^2,
\end{align*}
for some $\lambda > 0$.
\end{assumption}

By choosing $\nu= \delta(\cdot -\xi)$ in Assumption \ref{ass: convex} for $\xi \in B,$ we have $\mathrm{Cov}_{\nu}(\Phi(\cdot,\theta),\nabla_\theta \Phi(\cdot,\theta))=0$, which implies that $\ip{\nabla_\theta \Phi(x,\theta) -\nabla_\theta \Phi(x,\theta'),\theta-\theta'}\ge 2\lambda \norm{\theta-\theta'}^2$. Thus, the $2\lambda$-strong convexity of $G$ in $\theta$ also implies the $2\lambda$-strong convexity of $\Phi$ in $\theta$.



The assumptions collected in this sections are fairly strong, they are satisfied in certain linear-quadratic problems on bounded domains. We illustrate this in an example below.

\begin{example} \label{exampl}
We consider a prototypical adversarial robustness problem based on the potential $\Phi(\xi, \theta) := \norm{\xi-\theta}^2$ with $\theta$ in a bounded set $X' \subseteq X$ -- problems of this form appear, e.g., in adversarially robust linear regression. Next, we are going to verify that this problem satisfies Assumptions \ref{ass: lip} and \ref{ass: convex}.

We have $\nabla_\xi\Phi(\xi,\theta)=2(\xi-\theta),$ which is Lipschitz  in both $\theta$ and $\xi$. Since 
\begin{align*}
\nabla_\theta\Phi(\xi,\theta)&=2(\xi-\theta), \\
    \Phi(\xi,\theta)-\int_B\Phi(\xi,\theta)\nu(\ud \xi)&=(\norm{\xi}^2-\int_B \norm{\xi}^2\nu(\ud \xi)) -2\theta\cdot (\xi-\int_B \xi\nu(\ud \xi)) \\
    \nabla_\theta\Phi(\xi,\theta)-\int_B\nabla_\theta\Phi(\xi,\theta)\nu(\ud \xi) &=-2(\xi-\int_B \xi\nu(\ud \xi)),
\end{align*} we have that 
\begin{align*}
    G(\theta,\nu)= 2\theta-2\E_\nu[\xi]+4\theta\cdot \mathrm{Var}_\nu (\xi)-2 \mathrm{Cov}_\nu(\norm{\xi}^2,\xi),
\end{align*}
where $\E_\nu[\xi]=\int_B \xi\nu(\ud \xi)$  and $\mathrm{Cov}_\nu(\norm{\xi}^2,\xi)= \int_B(\norm{\xi}^2-\E_\nu[\norm{\xi}^2]) (\xi-\E_\nu[\xi])\nu(\ud \xi)$. Since the $\e$-ball and $\theta \in X'$ are bounded, we have that $G(\theta,\nu)$ is Lipschitz in both $\theta$ and $\nu$. Thus, it satisfies Assumption \ref{ass: lip}. 
In order to make $G(\theta,\nu)$ satisfy Assumption \ref{ass: convex}, we choose $\e$ small enough such that the term $4\theta\cdot \mathrm{Var}_\nu (\xi)$ is $1$-Lipschitz. In this case, we can verify that $\ip{G(\theta,\nu)-G(\theta',\nu),\theta-\theta'}\ge \norm{\theta-\theta'}^2$ and, thus,  Assumption \ref{ass: convex}.

\end{example}

\section{Propagation of chaos} \label{Sec_propChaos}
We now study the large particle limit ($N\to\infty$) of the Abram dynamics $\eqref{main}$. When considering a finite time interval $[0,T]$, we see that  the particle system $\eqref{main}$ approximates the McKean-Vlasov SDE \eqref{eq: limit} in this limit. We note that we assume in the following that $0< \varepsilon < 1$. Moreover, we use the Wasserstein-$2$ distance instead of Wasserstein-$1$ distance in Assumption \ref{ass: lip}. We have $\cW_1(\nu,\nu')\le \cW_2(\nu,\nu')$ for any probability measures $\nu,\nu'$ for which the distances are finite, see \cite{Villani}. Thus, convergence in $\cW_2$ also implies convergence in $\cW_1$. We now state the main convergence result.

\begin{theorem}\label{th: propag}
    Let Assumption  \ref{ass: lip} hold. Then, there is a constant $C_{d,T}>0$ such that for all $T\ge 0$ and $N\ge 1$ we have the following inequality
    \begin{align*}
        \sup_{t\in[0,T]}\E\big[\norm{\theta_t^N-\theta_t}^2+\cW^2_2(\mu_t^N,\mu_t)\big]\le C_{d,T}N^{-\alpha_d},
    \end{align*}
    where $\alpha_d= 2/d$ for $d>4$ and $\alpha_d=1/2$ for $d\le 4.$
\end{theorem}

Hence, we obtain convergence of both the gradient flow approximation $(\theta^N_t)_{t \geq 0}$ and the particle approximation $(\mu_t^N)$ to the respective components in the McKean-Vlasov SDE.
We  prove this result by a coupling method. To this end, we first collect a few auxiliary results: studying the large sample limit of an auxiliary particle system and the distance of the original particle system to the auxiliary system. 
To this end, we sample $N$ trajectories of $(\xi_t)_{t \geq 0}$ from equations (\ref{eq: limit}) as 
\begin{align}\label{eq: sample}
    \xi^i_t &=\xi^i_0+ \int_0^t\nabla_x \Phi(\xi^i_s,\theta_s)\mathrm{d}s + \sqrt{2}W^i_t+\int_0^tn(\xi^i_s)\mathrm{d}l^i_s \qquad (i=1, \ldots, N),
\end{align}
where the Brownian motions $(W_t^1,\dots,W_t^N)_{t \geq 0}$ are the ones from \eqref{main}.
Of course these sample paths $(\xi_t^1,\ldots,\xi_t^N)_{t \geq 0}$ are different from the $(\xi_t^{1,N},\ldots,\xi_t^{N,N})_{t \geq 0}$ in equation \eqref{main}: Here, $(\theta_t)_{t \geq 0}$ only depends on the law of $(\xi_t)_{t \geq 0}$ whereas $(\theta_t^N)_{t \geq 0}$ depends on position of the particles $(\xi^{i,N}_t)_{t \geq 0}$.
As the $(\xi_t^1)_{t \geq 0}, \ldots, (\xi_t^N)_{t \geq 0}$ are i.i.d., we can apply the empirical law of large numbers from \cite{Fournier1} and get the following result.
\begin{proposition}\label{prop: important} 
     Let Assumption \ref{ass: lip} hold. Then, 
 $$\sup_{t\ge 0}\E\Big[ \cW^2_2\Big(N^{-1}\sum^N_{i=1}\delta_{\xi_t^i},\mu_t\Big)\Big]\le C_d N^{-\alpha_d}.$$
\end{proposition}
For any $i = 1,\ldots,N$, we are now computing bounds for the pairwise distances between $\xi^i_t$ and $\xi_t^{i,N}$ for $t \geq 0$. We note again that these paths are pairwise coupled through the associated Brownian motions $(W_t^i)_{t \geq 0}$, respectively.
\begin{lemma}\label{lem: gron x}
Let Assumption \ref{ass: lip} hold. Then,
\begin{align*}
    \norm{\xi_t^{i,N}-\xi_t^i}^2\le 2L\int_0^t\Big[\norm{\xi_s^{i,N}-\xi_s^i}^2+ \norm{\theta_s^N-\theta_s}^2\Big] \mathrm{d}s \qquad (i = 1,\ldots,N),
\end{align*}
for $t \in [0,T]$.
\end{lemma}

\begin{proof}
We apply It\^o's formula to $\norm{\xi_t^{i,N}-\xi_t^i}^2$,
\begin{align*}
    \norm{\xi_t^{i,N}-\xi_t^i}^2= 2 \underbrace{\int_0^t \ip{\xi_s^{i,N}-\xi_s^i,\nabla_x \Phi(\xi^{i,N}_s,\theta^N_s)-\nabla_x \Phi(\xi^i_s,\theta_s)}\mathrm{d}s}_{(I1)}\\
    + \underbrace{2\int_0^t\ip{n(\xi^{i,N}_s), \xi_s^{i,N}-\xi_s^i }\mathrm{d}l^{i,N}_s -2\int_0^t\ip{n(\xi^i_s), \xi_s^{i,N}-\xi_s^i }\mathrm{d}l^i_s}_{(I2)}.
\end{align*}
 We first argue that $(I2)\le 0.$ Recall that $ n(x)=-x/\norm{x}$ and that the processes $(\xi^{i,N}_t)_{t \geq 0}$ and $(\xi^{i}_t)_{t \geq 0}$ take values in the $\varepsilon$-ball $B$ with $\varepsilon < 1$. Then, we have 
 \begin{align*}
     2\int_0^t\ip{n(\xi^{i,N}_s), \xi_s^{i,N}-\xi_s^i }\mathrm{d}l^{i,N}_s&=  2\int_0^t\ip{n(\xi^{i,N}_s), \xi_s^{i,N}}\mathrm{d}l^{i,N}_s- 2\int_0^t\ip{n(\xi^{i,N}_s), \xi_s^i }\mathrm{d}l^{i,N}_s\\
     &= -2 \e l^{i,N}_t- 2 \int_0^t\ip{n(\xi^{i,N}_s), \xi_s^i }\mathrm{d}l^{i,N}_s\le -2\e l^{i,N}_t+2\e\int_0^t\mathrm{d}l^{i,N}_s= 0.
 \end{align*}
 Similarly, we have
 \begin{align*}
     -2\int_0^t\ip{n(\xi^i_s), \xi_s^{i,N}-\xi_s^i }\mathrm{d}l^i_s = 2\int_0^t\ip{n(\xi^i_s), \xi_s^i- \xi_s^{i,N }}\mathrm{d}l^i_s \le 0.
 \end{align*}
 Hence, we have $(I2)\le 0.$ 

 For $(I1)$, due to Assumption \ref{ass: lip} and, again, due to the boundedness of $B,$ we have
 \begin{align*}
     (I1)&\le L\int_0^t\norm{\xi_s^{i,N}-\xi_s^i}\Big[\norm{\xi_s^{i,N}-\xi_s^i}+ \norm{\theta_s^N-\theta_s}\Big] \mathrm{d}s \le  2L\int_0^t\Big[\norm{\xi_s^{i,N}-\xi_s^i}^2+ \norm{\theta_s^N-\theta_s}^2\Big] \mathrm{d}s.
 \end{align*}
\end{proof}
Finally, we study the distance between $\theta_t^N$ and $\theta_t$ for $t \geq 0$.
\begin{lemma}\label{lem: gron theta}
   Let Assumption \ref{ass: lip} hold. Then, we have
    \begin{align*}
    \norm{\theta_t^N-\theta_t}^2\le  3L \int_0^t\norm{\theta_s^N-\theta_s}^2 \mathrm{d}s+\frac{2L}{N}\sum^N_{i=1}\int_0^t \norm{\xi^{i,N}_s-\xi_s^i}^2\mathrm{d}s+2L\int_0^t\cW^2_2(N^{-1}\sum^N_{i=1}\delta_{\xi_s^i},\mu_s)\mathrm{d}s,
    \end{align*}
    for $t \in [0,T].$
\end{lemma}
\begin{proof}
Due to Assumption \ref{ass: lip} and since $\cW_1(\mu^N_s,\mu_s)\le\cW_2(\mu^N_s,\mu_s)$, we have
    \begin{align}\label{ineq: thetaN}
         \norm{\theta_t^N-\theta_t}^2 =& -2\int_0^t \ip{\theta_s^N-\theta_s,G(\theta_s^N,\mu^N_s)-G(\theta_s,\mu_s)}\mathrm{d}s\nonumber\\
         &\le 2L\int_0^t \norm{\theta_s^N-\theta_s}\Big(\norm{\theta_s^N-\theta_s}+\cW_2(\mu^N_s,\mu_s)\Big)\mathrm{d}s\nonumber\\
         &\le 2L \int_0^t\norm{\theta_s^N-\theta_s}^2 \mathrm{d}s+2L\int_0^t \norm{\theta_s^N-\theta_s}\cW_2(\mu^N_s,\mu_s)\mathrm{d}s\nonumber\\
         &\le 3L \int_0^t\norm{\theta_s^N-\theta_s}^2 \mathrm{d}s+L\int_0^t \cW^2_2(\mu^N_s,\mu_s)\mathrm{d}s.
    \end{align}
The triangle inequality implies that
    \begin{align}\label{ineq: tringale}
        \cW^2_2(\mu^N_s,\mu_s)\le 2\cW^2_2(\mu^N_s,N^{-1}\sum^N_{i=1}\delta_{\xi_s^i})+2\cW^2_2(N^{-1}\sum^N_{i=1}\delta_{\xi_s^i},\mu_s)\nonumber\\
        \le \frac{2}{N}\sum^N_{i=1}\norm{\xi^{i,N}_s-\xi_s^i}^2+2\cW^2_2(N^{-1}\sum^N_{i=1}\delta_{\xi_s^i},\mu_s).
    \end{align}
{Combining \eqref{ineq: thetaN} and \eqref{ineq: tringale}, we obtain}
    \begin{align*}
         \norm{\theta_t^N-\theta_t}^2
         \le  3L \int_0^t\norm{\theta_s^N-\theta_s}^2 \mathrm{d}s+\frac{2L}{N}\sum^N_{i=1}\int_0^t \norm{\xi^{i,N}_s-\xi_s^i}^2\mathrm{d}s+2L\int_0^t\cW^2_2(N^{-1}\sum^N_{i=1}\delta_{\xi_s^i},\mu_s)\mathrm{d}s.
    \end{align*}
\end{proof}
We now proceed to the proof of Theorem \ref{th: propag}.
\begin{proofof}{Theorem \ref{th: propag}}  
We commence by constructing an upper bound for $$u_t^N:=N^{-1}\sum^N_{i=1} \norm{\xi^{i,N}_t-\xi_t^i}^2+ \norm{\theta_t^N-\theta_t}^2.$$ From Lemma \ref{lem: gron x} and Lemma \ref{lem: gron theta}, we have 
\begin{align*}
    u_t^N&\le  5L\int_0^t u_s^N\mathrm{d}s+
    2L\int_0^t\cW^2_2(N^{-1}\sum^N_{i=1}\delta_{\xi_s^i},\mu_s)\mathrm{d}s.
\end{align*}
Gr\"onwall's inequality implies that
\begin{align*}
    u_t^N\le  2Le^{5Lt}\int_0^t\cW^2_2(N^{-1}\sum^N_{i=1}\delta_{\xi_s^i},\mu_s)\mathrm{d}s.
\end{align*}
According to Proposition \ref{prop: important}, we have
\begin{align*}
    \E[ u_t^N]\le 2Le^{5Lt}\int_0^t\E[\cW^2_2(N^{-1}\sum^N_{i=1}\delta_{\xi_s^i},\mu_s)]\mathrm{d}s\le  2C_d Le^{(1+5L)t}N^{-\alpha_d},
\end{align*}
    whereas  (\ref{ineq: tringale}) implies 
    \begin{align*}
     \norm{\theta_t^N-\theta_t}^2+  \cW^2_2(\mu^N_s,\mu_s)
        \le  u_t^N+2\cW^2_2(N^{-1}\sum^N_{i=1}\delta_{\xi_s^i},\mu_s).
    \end{align*}
   Therefore,
     \begin{align*}
        \sup_{t\in[0,T]}\E\big[\norm{\theta_t^N-\theta_t}^2+\cW^2_2(\mu_t^N,\mu_t)\big]\le  \sup_{t\in[0,T]}\E[u_t^N]+ \sup_{t\in[0,T]}\E\big[\cW^2_2(\mu_t^N,\mu_t)\big]\le C_{d,T}N^{-\alpha_d},
    \end{align*}
    where $C_{d,T}= 2C_d(1+ Le^{(1+5L)t}).$

\end{proofof}

\section{Longtime behaviour of the McKean-Vlasov process} \label{Sec_McVl}
Theorem~\ref{th: propag} implies that the gradient flow approximation in Abram $(\theta_t^N)_{t \geq 0}$ converges to the corresponding part of the McKean--Vlasov SDE $(\theta_t)_{t \geq 0}$ given in \eqref{eq: limit}. In this section, we show that this McKean-Vlasov SDE is able to find the minimiser $\theta_*$ of $F = \int \Phi(\xi,\cdot) \pi^{\gamma, \varepsilon}(\mathrm{d}\xi|\cdot).$ This, thus, gives us a justification to use Abram to solve the Bayesian adversarial robustness problem. We start by showing that $F$ admits a minimiser.

\begin{proposition} \label{prop:minis}
    Let Assumptions \ref{ass: lip}  and \ref{ass: convex} hold. Then,  $F$ 
    admits at least one minimiser in $X$. 
\end{proposition}
\begin{proof}
    We first argue that $F$ 
    is bounded below and obtains a minumum at some point $\theta_*.$ From Subsection~\ref{Subsc: Assump}, we already know that $\Phi(0,\theta)$ is $2\lambda$ strongly convex in $\theta$. Without loss of generality, we assume $\Phi(0,0)=0$ and $\nabla_\theta\Phi(0,0)=0$, that is $\Phi(0,\theta)$ reaches its minimum $0$ at $\theta_*=0.$ Since  $\Phi(\xi,\cdot)$ is $2\lambda$ strongly convex for any $\xi \in B$, we have that 
    \begin{align}\label{ineq: prop phi}
        \Phi(\xi,\theta)\ge \Phi(\xi,0)+ \theta\cdot \nabla_\theta \Phi(\xi,0)+\lambda \norm{\theta}^2.
    \end{align}
    Assumption \ref{ass: lip} implies that, $$\norm{\nabla_\theta \Phi(\xi,0)}=\norm{\nabla_\theta \Phi(\xi,0)-\nabla_\theta \Phi(0,0)}\le L\norm{\xi}\le L$$
     and 
     $$\abs{\Phi(\xi,0)}= \abs{\Phi(\xi,0)-\Phi(0,0)}\le \sup_{\zeta\in B}\norm{\nabla_\xi \Phi(\zeta,0)}\norm{\xi}\le (L+C_0)\norm{\xi}\le L+C_0,$$ where $C_0=\norm{\nabla_\xi \Phi(0,0)}.$ Therefore, we have $\Phi(\xi,\theta)\ge -L-C_0- L\norm{\theta}+\lambda \norm{\theta}^2,$ which is bounded below by $-L-C_0-\frac{L^2}{4\lambda}.$ {Thus,}  $F$ 
  {is bounded below by the same value}. {We can} always choose some $R_0=R_0(L,\lambda,C_0)$, such that for $\norm{\theta}\ge R_0,$ $ \Phi(\xi,\theta)\ge C_0+L$. Moreover, we already have $ \Phi(\xi,0)\le L+C_0.$ Thus, $F(\theta) 
     \ge C_0+L$ when $\norm{\theta}\ge R_0$ and  
     $F(0) \le C_0+L$. Hence $F$ 
     attains its minimum on the $R_0$-ball $\{\theta \in X: \norm{\theta}\le R_0\}.$ 
\end{proof}
Before  stating the main theorem of this section -- the convergence of the McKean-Vlasov SDE to the minimiser of $F$ -- we need to introduce  additional assumptions. 
\begin{assumption}[Neumann Boundary Condition]\label{ass: neumann}
    Let $\Phi(\cdot, \theta)$ satisfy a Neumann boundary condition on $\partial B,$ 
    \begin{align*}
        \frac{\partial_\xi \Phi(\xi,\theta)}{\partial n}= \nabla_\xi\Phi(\xi,\theta)\cdot n(\xi)=0,
    \end{align*}
    for any $\theta\in X.$
\end{assumption}
For a general function $\Phi$ defined on $B$, this assumption can be satisfied by smoothly extending $\Phi$ on $B'$ with radius $2\e$ such that it vanishes near the  boundary of $B'$.
We shall see that this assumption guarantees the existence of the invariant measure of the auxiliary dynamical system \eqref{eq: coupling limit} that we introduce below.

\begin{assumption}[Small-Lipschitz]\label{ass: slip}
 For any probability measures $\nu$, $\tilde\nu$ on $B$ and $\theta\in \R^n,$
\begin{align*}
    \norm{G(\theta,\nu)-G(\theta,\tilde\nu)}\le \ell\norm{\nu-\tilde \nu}_{\TV},
\end{align*}
where $\ell= (\frac{(\delta\land \lambda)\sqrt{\lambda}e^{-t_0} }{4\sqrt{2}CL})\land (\frac{\sqrt{\lambda}}{\sqrt{2}L})$ 
and $t_0= t_0(\delta,\lambda,C)=(\delta\land \lambda)^{-1}\log(4C)$. The constants $\delta$ and $C$ appear in Proposition \ref{prop: FY2}.
\end{assumption}
 Equivalently, we may say this assumption requires $G$ to have a small enough Lipschitz constant. If $\e$ (the radius of $B$) is very small, this assumption is implied by  Assumption \ref{ass: lip}, since
     $\cW_1(\nu,\tilde \nu)\le \e^d \int_B  \int_B \boldsymbol{1}_{x\ne y}\pi(\ud x,\ud y)= \e^d \norm{\nu-\tilde\nu}_{\TV}.$

     We illustrate these assumptons again in the linear-quadratic problem that we considered in Example~\ref{exampl} and show that Assumptions \ref{ass: neumann} and \ref{ass: slip} can be satisfied in this case.
     \begin{example}[Example~\ref{exampl} continued]
We consider again $\Phi(\xi,\theta) = \norm{\xi-\theta}^2$ with $\theta$ in a bounded $X' \subseteq X$.
Unfortunately, $\Phi$ does not satisfy Assumption \ref{ass: neumann}, since  the term $(\xi-\theta)\cdot\xi$ is not necessary to be $0$ on the boundary of $B$. Instead, we study a slightly larger ball by considering $\hat{\varepsilon} = 2\varepsilon$ instead of $\varepsilon$ and also replace $\Phi$ by $\hat{\Phi}(x,\theta)=\norm{m(\xi)-\theta}^2,$ where   $m:\R^d\to \R^d$ is smooth and equal to $\xi$ on the $\e$-ball and vanishes near the boundary of the $2\e$-ball. Since $m(x)$ varnishes near the boundary of $2\e$-ball, $\hat{\Phi}$ satisfies Assumption \ref{ass: neumann}.

We note that $\nabla_\xi\hat{\Phi}(\xi,\theta)=2D_\xi m(\xi) (m(\xi)-\theta).$ Hence, $\nabla_\xi\hat{\Phi}$ is Lipschitz in both $\theta$ and $\xi$ which directly follows from the boundedness and Lipschitz continuity of $m$, $D_\xi m$. Analogously to Example~\ref{exampl}, we have 
\begin{align*}
    G(\theta,\nu)= 2\theta-2\E_\nu[m(\xi)]+4\theta\cdot \mathrm{Var}_\nu (m(\xi))-2 \mathrm{Cov}_\nu(\norm{m(\xi)}^2,m(\xi))
\end{align*}
         and also see that it still satisfies Assumptions \ref{ass: lip}, \ref{ass: convex} when $\theta$ is bounded and $\e$ is small. 
Finally, Assumption \ref{ass: slip} is satisfied if $\e$ is chosen to be sufficiently small.
     \end{example}

We are now able to state the main convergence theorem of this section. Therein, we still consider $\theta_*$ to be a minimizer of function of $F$. 
\begin{theorem}\label{th: limit}
      Let Assumptions \ref{ass: lip}, \ref{ass: convex}, \ref{ass: neumann}, and \ref{ass: slip} hold and  let $(\theta_t, \mu_t)_{t \geq 0}$ be the solution to the McKean-Vlasov SDE \eqref{eq: limit}. Then, there are constants $\eta >0$ and $\tilde C > 0$ with which we have
      \begin{align}\label{ineq: limit con}
        \norm{\theta_t-\theta_*}^2+ \norm{ \mu_t- \pi^{\gamma, \varepsilon}(\cdot|{\theta_*})}^2_{\TV}\le \tilde C \Big(\norm{\theta_0-\theta_*}^2+ \norm{ \mu_0- \pi^{\gamma, \varepsilon}(\cdot|{\theta_*})}^2_{\TV}\Big) e^{-\eta t}.
      \end{align}
\end{theorem} We can see this result as both a statement about the convergence of $(\theta_t^N)_{t \geq 0}$ to the minimiser, but also as an ergodicity statement about $(\theta_t^N,\xi_t)_{t \geq 0}$.
 The ergodicity of a McKean-Vlasov SDE with reflection has also been subject of Theorem 3.1 in \citet{FY2}. In their work, the process is required to have a non-degenerate diffusion term. Hence, their result does not apply immediately, since the marginal  $(\theta_t)_{t \geq 0}$ is  deterministic {(conditionally on $(\xi_t)_{t \geq 0}$)}. Our proof ideas, however, are still influenced by \cite{FY2}.

    We note additionally that Theorem \ref{th: limit} implies the uniqueness of the minimiser $\theta_*$ -- we had only shown existence in Proposition~\ref{prop:minis}: If there exists another minimizer $\theta_*'$, then the dynamics \eqref{eq: limit} is invariant at $(\theta_0,\xi_0)\sim \delta_{\theta_*'}\otimes \pi^{\gamma, \varepsilon}(\cdot|{\theta_*'})$, which means $(\theta_t,\xi_t)\sim \delta_{\theta_*'} \otimes \pi^{\gamma, \varepsilon}(\cdot|{\theta_*'})$ for all $t\ge 0.$ Hence, we have $\norm{\theta'_*-\theta_*}\le \tilde C\norm{\theta'_*-\theta_*}e^{-\eta t}$. The right-hand side vanishes as $t\to \infty,$ which implies $\theta_*'=\theta_*.$

 In order to prove Theorem \ref{th: limit}, we first consider the case where $\theta_t\equiv\theta_*$, i.e.,
\begin{align}\label{eq: coupling limit}
    \hat \xi_t &=\xi_0+ \int_0^t\nabla_x \Phi(\hat \xi_s,\theta_*)\mathrm{d}s + \sqrt{2}W_t+\int_0^t\nu(\hat \xi_s)\mathrm{d}\hat l_s.
\end{align}
We denote   the law of $\hat \xi_t$ by $\hat \mu_t$, $t \geq 0$. Motivated by \cite{FY2}, we first show the exponential ergodicity for the process $(\hat \xi_t)_{t \geq 0}$.
\begin{proposition}\label{prop: FY2}
     Let Assumptions \ref{ass: lip} and \ref{ass: neumann} hold. Then, $(\hat \xi_t)_{t \geq 0}$ defined in \eqref{eq: coupling limit} is well-posed and admits an unique invariant measure $\pi^{\gamma, \varepsilon}(\cdot|{\theta_*}).$ Moreover, $(\hat \xi_t)_{t \geq 0}$ is exponentially ergodic. In particular, there exist $C,\delta>0,$ such that
     \begin{align*}
          \norm{\hat \mu_t- \pi^{\gamma, \varepsilon}(\cdot|{\theta_*})}^2_{\TV}\le C \norm{\mu_0- \pi^{\gamma, \varepsilon}(\cdot|{\theta_*})}^2_{\TV} e^{-\delta t}.
     \end{align*}
    
\end{proposition}
\begin{proof}
 The well-posedness and exponential ergodicity is a direct corollary of \citet[Theorem 2.3]{FY2}.   We only need to verify that $\pi^{\gamma, \varepsilon}(\cdot|{\theta_*})$ is invariant under the dynamics \eqref{eq: coupling limit}. We know that the probability distributions $(\hat \mu_t)_{t \geq 0}$ satisfies the following linear PDE with Neumann boundary condition 
 \begin{align*}
    \partial_t \hat \mu_t = \Delta \hat \mu_t-\mathrm{div} (\hat\mu_t\nabla_\xi\Phi(\xi,\theta_*)),\qquad   \frac{\partial \hat \mu_t}{\partial n}\Big|_{\partial B}=0.
 \end{align*}
 So any invariant measure of the dynamics \eqref{eq: coupling limit} is a probability distribution that  solves the following stationary PDE
 \begin{align*}
     \Delta \hat \mu-\mathrm{div} (\hat\mu\nabla_\xi\Phi(\xi,\theta_*))=0,\qquad \frac{\partial \hat \mu}{\partial n}\Big|_{\partial B}=0.
 \end{align*}
Now, $\hat\mu=\pi^{\gamma, \varepsilon}(\cdot|{\theta_*})$ is a basic result in the theory of Langevin SDEs on with reflection, see, e.g., \cite{Sato}.
\end{proof}

Most of the time, we are not able to quantify the constants $C$ and $\delta$: the Harris-like theorem from \cite{FY2}  is not quantitative. A special case in which we can quantify $C$ and $\delta$ is when the potential separates in the sense that $\nabla_\xi\Phi(\xi, \theta_*)= (f_1(\xi_1, \theta_*),...,f_{d_Y}(\xi_{d_Y}, \theta_*))$. 
Then  \eqref{eq: coupling limit} can be viewed as $d_Y$ independent reflection SDEs. If we denote their ergodicity constants as $C_i$ and $\delta_i$ for $i=1,...,d_Y$, then \citet[Proof of Proposition 1]{Jin1} implies that we can choose $C:=\sum_{i=1}^d C_i$ and $\delta:=\min_{i=1,...,d}\delta_i$.

 Next, we bound the distance $\norm{\mu_t- \hat \mu_t}_{\TV}$ by Girsanov’s theorem --  a classical way to estimate the distance between two SDEs with different drift terms. This is again motivated by \citet[proof of Lemma 3.2]{FY2}. There, the method is used to bound the distance between two measure-dependent SDEs. In our case, it also involves the state $\theta_t$, which depends on $t.$ Hence, the right-hand side depends on the path of $(\theta_s)_{0\le s\le t}.$
\begin{lemma}\label{lem: hatmu}
Let Assumption \ref{ass: lip} hold. Then, we have
\begin{align*}
    \norm{\mu_t- \hat \mu_t}^2_{\TV}\le L^2 \int_0^t\norm{\theta_s-\theta_*}^2 \mathrm{d}s.
\end{align*}
\end{lemma}

\begin{proof}
 We follow the same idea as \citet[proof of Lemma 3.2]{FY2}.  In our case, we need to choose 
\begin{align*}
        Z_t= \exp\Big(\int_0^t z(\theta_*,\theta_s,\xi_s)\cdot
        \mathrm{d}W_s-\frac{1}{2}\int_0^t \norm{z(\theta_*,\theta_s,\xi_s)}^2\mathrm{d}s \Big),
    \end{align*}
  where $z(\theta_*,\theta,x)= (\nabla_x \Phi(x,\theta)-\nabla_x \Phi(x,\theta_*))/\sqrt{2}$. \cite[Proposition 5.6]{legall} implies that  the process  $(Z_t)_{t \geq 0}$ is a martingale due to  $z(\theta_*,\theta_s,\xi_s)$ being bounded and  $\int_0^t \norm{z(\theta_*,\theta_s,\xi_s)}^2\mathrm{d}s$ being the quadratic variation process of $\int_0^t z(\theta_*,\theta_s,\xi_s)\cdot
        \mathrm{d}W_s$. 

We define the probability measure $\mathbb{Q}_t:= Z_t\mP,$ i.e. $\mathbb{Q}_t(A):= \E[ Z_t \boldsymbol{1}_A]$ for any $\mathcal{F}_t$-measurable set $A$.  And we notice that the quadratic covariation between $\int_0^t z(\theta_*,\theta_s,\xi_s)\cdot
        \mathrm{d}W_s$ and $W_t$ is given by
\begin{align*}
    \ip{\int_0^. z(\theta_*,\theta_s,\xi_s)\cdot
        \mathrm{d}W_s,W_.}_t= \int_0^t z(\theta_*,\theta_s,\xi_s)\mathrm{d}s.
\end{align*}
Hence by Girsanov’s theorem (see \cite[Theorem 5.8, Cons\'equences (c)]{legall}), $\tilde W_t:= W_t-\int_0^t z(\theta_*,\theta_s,\xi_s)\mathrm{d}s$ is a Brownian motion under $\mathbb{Q}_t$ with the same filtration $\mathcal{F}_t$.

      We rewrite (\ref{eq: limit}) as
    \begin{align*}
         \xi_t &=\xi_0+ \int_0^t\nabla_x \Phi(\xi_s,\theta_*)\mathrm{d}s + \sqrt{2}\tilde W_t+\int_0^tn(\xi_s)\mathrm{d}l_s,
    \end{align*}
    which has the same distribution as $\hat \xi_t$ under $\mathbb{Q}_t.$
    Hence  
    \begin{align*}
        \norm{\mu_t- \hat \mu_t}_{\TV}=& \sup_{\abs{f}\le 1}\abs{\E[f(\xi_t)]- \E[f(\xi_t)Z_t]}\le \E[\abs{Z_t-1}]\\
     \le& 2\E[R_t\log(R_t)]^{\frac{1}{2}}=2 \E_{Q_t}\Big[\int_0^t z(\theta_*,\theta_s,\xi_s)\cdot\mathrm{d}W_s-\frac{1}{2}\int_0^t \norm{z(\theta_*,\theta_s,\xi_s)}^2\mathrm{d}s\Big]^{\frac{1}{2}}\\
        =& 2 \E_{Q_t}\Big[\int_0^t z(\theta_*,\theta_s,\xi_s)\cdot\mathrm{d}\tilde W_s+\frac{1}{2}\int_0^t \norm{z(\theta_*,\theta_s,\xi_s)}^2\mathrm{d}s\Big]^{\frac{1}{2}}\\
        =&  \sqrt{2} \E_{Q_t}\Big[\int_0^t \norm{z(\theta_*,\theta_s,\xi_s)}^2\mathrm{d}s\Big]^{\frac{1}{2}}\le L\Big(\int_0^t\norm{\theta_s-\theta_*}^2 \mathrm{d}s\Big)^{\frac{1}{2}}
    \end{align*}
    where the first ``$\leq$'' is implied by Pinsker's inequality.
\end{proof}

Using these auxiliary results, we can now formulate the proof of Theorem \ref{th: limit}. 

\begin{proofof}{Theorem \ref{th: limit}} 
   We take the time derivative of $\norm{\theta_t-\theta_*}^2,$
    \begin{align*}
        \frac{\mathrm{d}\norm{\theta_t-\theta_*}^2}{\mathrm{d}t}=& - \ip{G(\theta_t,\mu_t)-G(\theta_*,\pi^{\gamma, \varepsilon}(\cdot|{\theta_*})),\theta_t-\theta_*}\\
        &\le -2\lambda \norm{\theta_t-\theta_*}^2+\ell \norm{\theta_t-\theta_*}\norm{\mu_t- \pi^{\gamma, \varepsilon}(\cdot|{\theta_*})}_{\TV}\\
        &\le  -\lambda \norm{\theta_t-\theta_*}^2+\frac{\ell^2}{\lambda} \norm{\mu_t- \pi^{\gamma, \varepsilon}(\cdot|{\theta_*})}^2_{\TV},
    \end{align*}
    where the first ``$\leq$'' is due to the $\varepsilon$-Young's inequality.
    This implies $$\frac{\mathrm{d}(e^{\lambda t}\norm{\theta_t-\theta_*}^2)}{\mathrm{d}t}\le \frac{\ell^2}{\lambda}e^{\lambda t} \norm{\mu_t- \pi^{\gamma, \varepsilon}(\cdot|{\theta_*})}^2_{\TV},$$
 Hence, we have
 \begin{align}\label{ineq: theta-theta}
     \norm{\theta_t-\theta_*}^2\le e^{-\lambda t}\norm{\theta_0-\theta_*}^2+\frac{\ell^2}{\lambda}\int_0^t  \norm{\mu_s- \pi^{\gamma, \varepsilon}(\cdot|{\theta_*})}^2_{\TV}\mathrm{d}s.
 \end{align}
Then, using the triangle inequality, we see that
   \begin{align*}
       \norm{\theta_t-\theta_*}^2+ m\norm{ \mu_t- \pi^{\gamma, \varepsilon}(\cdot|{\theta_*})}^2_{\TV}&\le \underbrace{ \norm{\theta_t-\theta_*}^2}_{\eqref{ineq: theta-theta}}+\underbrace{2m\norm{ \mu_t- \hat\mu_t}^2_{\TV}}_{\text{Lemma }  \ref{lem: hatmu}}+\underbrace{2m\norm{ \hat\mu_t- \pi^{\gamma, \varepsilon}(\cdot|{\theta_*})}^2_{\TV}}_{\text{Proposition }  \ref{prop: FY2}}
      \\ &\le \int_0^t \Big(\frac{\ell^2}{\lambda}\norm{\mu_s-\pi^{\gamma, \varepsilon}(\cdot|{\theta_*})}^2_{\TV}+2mL^2\norm{\theta_s-\theta_*}^2\Big)\mathrm{d}s\\
       &\qquad +2C\Big(\norm{\theta_0-\theta_*}^2+m\norm{\mu_0- \pi^{\gamma, \varepsilon}(\cdot|{\theta_*})}^2_{\TV}\Big)e^{-(\delta\land \lambda) t}.
   \end{align*}
   Let $m=m(\ell,L,\lambda)=\frac{\ell}{L\sqrt{2\lambda}}$, we conclude from the above inequality that
   \begin{align*}
       \norm{\theta_t-\theta_*}^2+ m\norm{ \mu_t- \pi^{\gamma, \varepsilon}(\cdot|{\theta_*})}^2_{\TV}&\le 2mL^2\int_0^t \Big(m\norm{\mu_s-\pi^{\gamma, \varepsilon}(\cdot|{\theta_*})}^2_{\TV}+\norm{\theta_s-\theta_*}^2\Big)\mathrm{d}s\\
        &\qquad +2C\Big(\norm{\theta_0-\theta_*}^2+m\norm{\mu_0- \pi^{\gamma, \varepsilon}(\cdot|{\theta_*})}^2_{\TV}\Big)e^{-(\delta\land \lambda) t}.
   \end{align*}
  Hence, by Gr\"onwall's inequality, we have
  \begin{align}\label{ineq: long}
    \norm{\theta_t-\theta_*}^2&+ m\norm{ \mu_t- \pi^{\gamma, \varepsilon}(\cdot|{\theta_*})}^2_{\TV} \nonumber \\ &\le  C\Big(\norm{\theta_0-\theta_*}^2+m\norm{\mu_0- \pi^{\gamma, \varepsilon}(\cdot|{\theta_*})}^2_{\TV}\Big) \Big(2mL^2\int_0^t e^{2mL^2(t-s)} e^{-(\delta\land \lambda) s}\mathrm{d}s+e^{-(\delta\land \lambda) t}\Big)\nonumber\\
    &= C\Big(\frac{2mL^2}{2mL^2+\delta\land\lambda}(e^{2mL^2 t}-e^{-(\delta\land \lambda)t})+e^{-(\delta\land \lambda) t}\Big)\Big(\norm{\theta_0-\theta_*}^2+m\norm{\mu_0- \pi^{\gamma, \varepsilon}(\cdot|{\theta_*})}^2_{\TV}\Big)\nonumber\\
    &\le C\Big(\frac{2mL^2}{\delta\land\lambda} e^{2mL^2 t}+e^{-(\delta\land \lambda) t}\Big)\Big(\norm{\theta_0-\theta_*}^2+m\norm{\mu_0- \pi^{\gamma, \varepsilon}(\cdot|{\theta_*})}^2_{\TV}\Big)\nonumber\\
    &\le C_{t, L,\ell,\lambda}\Big(\norm{\theta_0-\theta_*}^2+m\norm{\mu_0- \pi^{\gamma, \varepsilon}(\cdot|{\theta_*})}^2_{\TV}\Big),
  \end{align}
 where $C_{t, L,\ell,\lambda}= C\Big(\frac{2mL^2}{\delta\land\lambda} e^{2mL^2 t}+e^{-(\delta\land \lambda) t}\Big).$ According to Assumption \ref{ass: slip}, we know that $2mL^2=\frac{\ell L\sqrt{2}}{\sqrt{\lambda}}\le 1$. Next, we are going to show that $0<C_{t_0, L,\ell,\lambda}\le 1/2$ for  $t_0=t_0(\delta,\lambda,C)=(\delta\land \lambda)^{-1}\log(4C).$ Again, from Assumption \ref{ass: slip}, we know $\frac{2mL^2}{\delta\land\lambda} e^ {t_0}\le \frac{1}{4C}$. Hence we finally have,
\begin{align*}
    C_{t_0, L,\ell,\lambda}\le C\frac{2mL^2}{\delta\land\lambda} e^ {t_0}+ Ce^{-(\delta\land \lambda) t_0}\le \frac{1}{2}.
\end{align*}
For any $t\ge 0,$ we always have $[\frac{t}{t_0}]t_0\le t < [\frac{t}{t_0}]t_0+t_0,$ where $[x]$ denotes  the greatest integer  $\leq x.$ Hence,
\begin{align*}
    \norm{\theta_t-\theta_*}^2+ m\norm{ \mu_t- \pi^{\gamma, \varepsilon}(\cdot|{\theta_*})}^2_{\TV}&\le 2^{-[\frac{t}{t_0}]}\Big(\norm{\theta_{t- [\frac{t}{t_0}]t_0}-\theta_*}^2+m\norm{\mu_{t- [\frac{t}{t_0}]t_0}- \pi^{\gamma, \varepsilon}(\cdot|{\theta_*})}^2_{\TV}\Big)\\
    &\le 2^{-\frac{t}{t_0}+1}\sup_{0\le s\le t_0}\Big(\norm{\theta_s -\theta_*}^2+m\norm{\mu_s- \pi^{\gamma, \varepsilon}(\cdot|{\theta_*})}^2_{\TV}\Big)\\
    &\le  2^{-\frac{t}{t_0}+1}C\Big(\frac{e^{t_0}}{\delta\land\lambda} +1\Big)\Big(\norm{\theta_0 -\theta_*}^2+m\norm{\mu_0- \pi^{\gamma, \varepsilon}(\cdot|{\theta_*})}^2_{\TV}\Big),
\end{align*}
where the last inequality is from \eqref{ineq: long} and $C_{s, L,\ell,\lambda}$ could be bounded by $C(\frac{e^{t_0}}{\delta\land\lambda} +1)$ for $0\le s\le t_0$. And since $m\le \frac{1}{2L^2}<1,$ we conclude that
\begin{align*}
     \norm{\theta_t-\theta_*}^2+ \norm{ \mu_t- \pi^{\gamma, \varepsilon}(\cdot|{\theta_*})}^2_{\TV}&\le m^{-1}2^{-\frac{t}{t_0}+1}C\Big(\frac{e^{t_0}}{\delta\land\lambda} +1\Big)\Big(\norm{\theta_0 -\theta_*}^2+\norm{\mu_0- \pi^{\gamma, \varepsilon}(\cdot|{\theta_*})}^2_{\TV}\Big)\\
     &\le 2L^2 2^{-\frac{t}{t_0}+1}C\Big(\frac{e^{t_0}}{\delta\land\lambda} +1\Big)\Big(\norm{\theta_0 -\theta_*}^2+\norm{\mu_0- \pi^{\gamma, \varepsilon}(\cdot|{\theta_*})}^2_{\TV}\Big).
\end{align*}
Finally, we choose the constants $\eta=\eta(\delta,\lambda,C)= \log(2)t_0^{-1}=(\delta\land \lambda)\frac{\log(2)}{\log(4C)}$ and $\tilde C= \tilde C(L,C,\delta,\lambda)= 4CL^2\Big(\frac{e^{t_0}}{\delta\land\lambda} +1\Big)= 4CL^2\Big(\frac{(4C)^{(\delta\land\lambda)^{-1}}}{\delta\land\lambda} +1\Big) .$
\end{proofof}

\section{Algorithmic considerations} \label{Sec_Discre}
Throughout this work, we have considered Abram as a continuous-time dynamical system. To employ it for practical adversarially robust machine learning, this system needs to be discretised, i.e., we need to employ a time stepping scheme to obtain a sequence $(\theta^N_k, \xi_k^{1, N}, \ldots,\xi_k^{N,N})_{k=1}^\infty$ that approximates Abram at discrete points in time. We now propose two  discrete schemes for Abram, before then discussing the simulation of  Bayesian adversarial attacks.

\paragraph{Discrete Abram.} We initialise the particles by sampling them from the uniform distribution in the $\varepsilon$-ball. Then, we  employ a projected Euler-Maruyama scheme to discretise the particles $(\xi_t^{1, N}, \ldots,\xi_t^{N,N})_{t \geq 0}$. The Euler-Maruyama scheme (see, e.g., \citealt{higham2021introduction}) is a standard technique for first order diffusion equations -- we use a projected version to adhere to the reflecting boundary condition inside the ball $B$. Projected Euler-Maruyama schemes of this form have been studied in terms of almost sure convergence \citep{SLOMINSKI1994197} and, importantly, also in terms of their longtime behaviour \citep{lamperski21a}. The gradient flow  part $(\theta^N_t)_{t \geq 0}$ is discretised using a forward Euler method -- turning the gradient flow into a gradient descent algorithm \citep{nocedal1999numerical}. In applications, it is sometimes useful to allow multiple iterations of the particle dynamics $(\xi_t^{1, N}, \ldots,\xi_t^{N,N})_{t \geq 0}$ per iteration of the gradient flow $(\theta^N_t)_{t \geq 0}$. This corresponds to a linear time rescaling in the particle dynamics that should lead to a more accurate representation of the respective adversarial distribution.  

If the number of data sets $(y_k, z_k)_{k=1}^K$ is large, we may be required to use a data subsampling technique. Indeed,  we approximate $\frac{1}{K}\sum_{k=1}^K \Phi(y_k, z_k|\theta) \approx \Phi(y_{k'}, z_{k'})$ with a $k' \sim \mathrm{Unif}(\{1,\ldots,K\})$ being sampled independently in every iteration of the algorithm. This gives us  a stochastic gradient descent-type approximation of the gradients in the algorithm, see \cite{RobbinsMonro}. We note that we have not analysed data subsampling within Abram -- we expect that techniques from \cite{Jin1, Latz} may be useful to do so.
We summarise the method in Algorithm~\ref{alg:BAT1}.

\begin{algorithm}[hptb]\caption{Abram}
\begin{algorithmic}[1]
  \STATE initialise learning rate $h$, $\theta_0$, $\gamma$, $\varepsilon$
 
  \FOR{$j = 1, 2,\ldots, J$}
     \STATE pick a data point $(y_j, z_j)$ from training data
  \FOR{$i=1,2,\ldots ,N$}
  \STATE initialise $\xi_{0, j}^i= \xi_{T, j-1}^i$ if $j>1$ else $\xi_{0, j}^i \sim \text{Unif}[-\e, \e]$ 

  \FOR{$\tau = 1, 2, \ldots, T$ }
  \STATE

  $\xi_{\tau, j}^i \leftarrow \mathrm{Proj}_{\|\cdot \| \leq \varepsilon} (\xi_{\tau-1, j}^i + h\nabla_\xi\Phi(y_j + \xi_{\tau-1, j}^i, z_j| \theta_{j-1}) + \gamma^{-1} \sqrt{2h}w_{\tau, j}^i) \quad  (w_{\tau, j}^i \sim \mathrm{N}(0,\mathrm{Id}) \text{ iid.})$

  \ENDFOR
  \ENDFOR
  \STATE ${\mu}_j^N \leftarrow \frac{1}{N} \sum_{i=1}^N\delta(\cdot - \xi_{T, j}^i)$
  \STATE $\hat{C}_j \leftarrow \mathrm{Cov}_{{\mu}_j^N}(\Phi(y_j + \cdot, z_j| \theta_{j-1}),\nabla_\theta\Phi(y_j + \cdot, z_j | \theta_{j-1}))$
  \STATE
 $\theta_j \leftarrow  \theta_{j-1} - \frac{h}{N}\sum_{i=1}^N\nabla_\theta\Phi(y_j + \xi_{T, j}^i, z_j|\theta_{j-1}) - \gamma h \hat{C}_j$ 
  \ENDFOR
  \RETURN $\theta_J$
\end{algorithmic}
\label{alg:BAT1}
\end{algorithm}
\paragraph{Discrete Abram with mini-batching.}
When subsampling in machine learning practice, it is usually advisable to choose mini-batches of data points rather than single data points. Here, we pick a mini-batch $\{y_{k'}, z_{k'}\}_{k' \in K'} \subseteq \{y_k, z_k\}_{k=1}^K$, with $\#K ' \ll K$ and perform the gradient step with all elements with index in $K'$ rather than a single element in the whole data set $\{y_k, z_k\}_{k=1}^K$. Abram would then require a set of $N$ particles for each of the elements in the batch, i.e., $NK'$ particles in total. In practice, $N$  and $K'$ are both likely to  be large, leading to Abram becoming computationally infeasible. Based on an idea discussed in a different context in \cite{Hanu}, we propose the following method: in every time step $j = 1,\ldots, J$ we choose an identical number of particles $(\xi^{i}_{T,j})_{i=1}^N$ and data sets $(y^{i}_j,z^{i}_j)_{i=1}^N$ in the mini-batch, i.e. $\# K' = N$. Then, we employ the Abram dynamics, but equip each particle $\xi^{i}_{T,j}$ with a different data point $(y^{i}_j,z^{i}_j)$ $(i = 1,\ldots,N)$. As opposed to Abram with separate particles per data point, we here compute the sampling covariance throughout all subsampled data points rather than separately for every data point. The resulting dynamics are then only close to \eqref{main}, if we assume that the adversarial attacks for each data point are not too dissimilar of each other. However, the dynamics may also be successful, if this is not the case. We summarise the resulting method in Algorithm~\ref{alg:BAT}.

\begin{algorithm}[hptb]\caption{Mini-batching Abram}
\begin{algorithmic}[1]
  \STATE initialize learning rate $h$, $\theta_0$, $\gamma$, $\varepsilon$
  \FOR{$j = 1, 2,\ldots, J$}
  \FOR{$i=1,2,\ldots ,N$}
  \STATE initialize $\xi_{0, j}^i \leftarrow \xi_{T, j-1}^i$ if $j>1$ else $\xi_{0, j}^i \sim \text{Unif}[-\e, \e]$ iid. 
  \STATE pick $N$ data points $(y_j^i, z_j^i)_{i=1}^N$ from the training data $(y_k, z_k)_{k=1}^K$
  \FOR{$\tau = 1, 2, \ldots, T$ }
  \STATE $
  \xi_{\tau, j}^i \leftarrow \mathrm{Proj}_{\|\cdot \| \leq \varepsilon} (\xi_{\tau-1, j}^i + h\nabla_\xi\Phi(y_j^i + \xi_{\tau-1, j}^i, z_j^i | \theta_{j-1}) + \gamma^{-1} \sqrt{2h}w_{\tau, j}^i) \qquad  (w_{\tau, j}^i \sim \mathrm{N}(0,\mathrm{Id}) \text{ iid.})$
  \ENDFOR
  \ENDFOR
   \STATE ${\mu}_j^N \leftarrow \frac{1}{N} \sum_{i=1}^N\delta(\cdot - (y_j^i + \xi_{T, j}^i))$
  \STATE $\hat{C}_j \leftarrow \mathrm{Cov}_{{\mu}_j^N}(\Phi(\cdot, z_j| \theta_{j-1}),\nabla_\theta\Phi(\cdot, z_j | \theta_{j-1}))$
  \STATE
 $\theta_j \leftarrow  \theta_{j-1} - \frac{h}{N}\sum_{i=1}^N\nabla_\theta\Phi(y_j^i + \xi_{T, j}^i, z_j|\theta_{j-1}) - \gamma h \hat{C}_j$ 
  \ENDFOR
  \RETURN $\theta_J$
\end{algorithmic}
\label{alg:BAT}
\end{algorithm}

\paragraph{Bayesian attacks.}
The mechanism used to approximate the Bayesian adversary in Algorithm \ref{alg:BAT1} can naturally be used as a Bayesian attack. We propose two different attacks:
\begin{enumerate}
    \item We use the projected Euler-Maruyama method to sample from the Bayesian adversarial distribution $\pi^{\gamma, \varepsilon}$ corresponding to an input data set $y \in Y$ and model parameter $\theta^*$. We summarise this attack in Algorithm \ref{alg:BA}.
    \item Instead of attacking with a sample from $\pi^{\gamma, \varepsilon}$, we can attack with the mean of said distribution.  From Proposition~\ref{prop: FY2}, we know that the particle system $(\hat{\xi}_t)_{t \geq 0}$ that is based on a fixed parameter $\theta_*$, is exponentially ergodic. Thus, we approximate the mean of $\pi^{\gamma, \varepsilon}$, by sampling $(\hat{\xi}_t)_{t \geq 0}$ using projected Euler-Maruyama and approximate the mean by computing the sample mean throughout the sampling path.  We summarise this method in Algorithm \ref{alg:BA attack mean}. 
\end{enumerate}

\begin{algorithm}[hptb]\caption{Bayesian sample attack}
\begin{algorithmic}[1]
  \REQUIRE unperturbed input data set $y$
  \STATE initialise $h$, $\gamma$, $\varepsilon$, $\xi_0 \sim \text{Unif}[-\e, \e]$
  \FOR{$j = 1, 2, \ldots, J$}
  \STATE
 $\xi_j \leftarrow \mathrm{Proj}_{\|\cdot \| \leq \varepsilon} (\xi_{j-1} + h\nabla_\xi\Phi(x + \xi_{j-1}, \theta) + \gamma^{-1} \sqrt{2h}w_j) \qquad  (w_j \sim \mathrm{N}(0,\mathrm{Id}))$
    \ENDFOR
  \RETURN adversarially perturbed input data point $y + \xi_J$
\end{algorithmic}
\label{alg:BA}
\end{algorithm}

\begin{algorithm}[hptb]\caption{Bayesian mean attack}
\begin{algorithmic}[1]
   \REQUIRE unperturbed input data point $y$
  \STATE initialise $h$, $\gamma$, $\varepsilon$, $\xi_0 \sim \text{Unif}[-\e, \e]$
  \FOR{$j = 1, 2, \ldots, J$}
  \STATE
 $\xi_j \leftarrow \mathrm{Proj}_{\|\cdot \| \leq \varepsilon} (\xi_{j-1} + h\nabla_\xi\Phi(x + \xi_{j-1}, \theta) + \gamma^{-1} \sqrt{2h}w_j) \qquad  (w_j \sim \mathrm{N}(0,\mathrm{Id}))$
    \ENDFOR
  \RETURN adversarially perturbed input data point $y +\frac{1}{J}\sum_{j=1}^J \xi_j$
\end{algorithmic}
\label{alg:BA attack mean}
\end{algorithm}

\section{Deep learning experiments} \label{Sec_Exp}

We now study the application of the Bayesian Adversary in deep learning.  The model parameter $\theta$ is updated for $J$ steps with batch size/number of particles $N$. For each particle in the ensemble, the perturbation parameter $\xi$ is updated for $T$ steps.

\subsection{MNIST}
We test Algorithm \ref{alg:BAT1} and Algorithm \ref{alg:BAT} on the classification benchmark data set MNIST \citep{LeCun2005TheMD} against different adversarial attacks and compare the results with the results after an FGSM-based \citep{Wong2020Fast} adversarial training. Each experimental run is conducted on a single Nvidia A6000 GPU. We utilize the Adversarial Robustness Toolbox (ART) for the experiments, see \cite{art2018} for more details. ART is a Python library for adversarial robustness that provides various APIs for defence and attack.  
We use a neural network with two convolution layers each followed by a max pooling. In Algorithm \ref{alg:BAT1}, we set $\gamma=1, h=\varepsilon, \varepsilon=0.2$.
In Algorithm \ref{alg:BAT}, we set $\gamma=1, h=10\varepsilon, \varepsilon=0.2$. We observe that setting larger noise scale for the attack during training helps Abram's final evaluation performance. We train the neural network for 30 epochs (i.e., 30 full iterations through the data set) for each method. The number of particles (and batch size) is $N=128$ and the inner loop is trained for $T=10$ times. To better understand how Abram responds to different attacks, we test against six attack methods: PGD \citep{madry2018towards},  Auto-PGD \citep{pmlr-v119-croce20b}, Carlini and Wagner \citep{7958570}, Wasserstein Attack \citep{pmlr-v97-wong19a}, as well as the Bayesian attacks introduced in this paper -- see Algorithms \ref{alg:BA} and \ref{alg:BA attack mean}. We also test the method's accuracy in the case of  benign (non-attacked) input data.  For the Bayesian sample attack and Bayesian mean attack, we set $\gamma=1000$. See Table~\ref{Table:deep_learning_test_acc_mnist} for the comparison. The results are averaged over three random seeds. 
We observe that Abram performs similarly to FGSM under  Wasserstein, Bayesian sample, and Bayesian mean attack. FGSM outperforms Abram under Auto-PGD, PGD, and Carlini \& Wagner attack. We conclude that Abram is as effective as FGSM under certain weaker attacks, but can usually not outperform the conventional FGSM. 

\begin{table}[h!]
    \centering 
	\begin{tabular}{rlccllc}
		\toprule
		&\textbf{Adversarial Attack ($\varepsilon=0.1$)} & \textbf{Abram} &  \textbf{Mini-batching Abram}  & \quad\quad\textbf{FGSM} &\\ \midrule
		&Benign Test &   92.41$\pm$0.05&  99.28$\pm$0.04   & \quad\quad 99.44$\pm$0.05  &     \\
            \hdashline
		&Auto-PGD 
 & 78.18$\pm$0.20 &    95.86$\pm$0.18      & \quad\quad 98.84$\pm$0.05  &   \\ 
            &PGD 
       & 78.24$\pm$0.17   &     95.86$\pm$0.17      & \quad\quad 98.85$\pm$0.04   &        \\ 
		&Wasserstein Attack 
  & 86.27$\pm$0.12 &  96.51$\pm$0.13     &\quad\quad  96.97$\pm$0.04    &      \\ 
		
		&Carlini \& Wagner Attack  
  & 8.76$\pm$0.015 &  5.14$\pm$0.1   &  \quad\quad 62.60$\pm$0.02& \\ 
		&Bayesian sample attack   & 92.43$\pm$0.10    & 99.29$\pm$0.03       &  \quad\quad 99.44$\pm$0.06    &    \\ 
		&Bayesian mean attack  &  92.42$\pm$0.08   & 99.28$\pm$0.04 &   \quad\quad 99.44$\pm$0.05  & \\
		\bottomrule
	\end{tabular}
	  \caption{\small Comparison of test accuracy (\%) on MNIST with different adversarial attack after  Abram, mini-batching Abram, and FGSM \citep{Wong2020Fast} adversarial training.}
	\label{Table:deep_learning_test_acc_mnist}
\end{table}
Another observation is that mini-batching Abram outperforms Abram significantly. Recall that in Abram we used $128$ particles for each data point which can be viewed as SGD with batch size $1$, whereas the mini-batching Abram is similar to the mini-batching SGD. Mini-batching Abram has the freedom to set the batch size which helps to reduce the  variance in the stochastic optimisation and, thus, gives more stable results. In particular, with mini-batching Abram, gradients are approximated by multiple data points instead of one data point which is the case in Abram. Having a larger batch size also increases computation efficiency by doing matrix multiplication on GPUs, which is important in modern machine learning applications as the datasets can be expected to be large.

\subsection{CIFAR10}

Similarly, we test Algorithm \ref{alg:BAT} on the classification benchmark dataset CIFAR10 \citep{Krizhevsky2009LearningML} by utilising ART. The dataset is pre-processed by random crop and random horizontal flip following \cite{Krizhevsky2009LearningML} for data augmentation. 
The neural network uses the Pre-act ResNet-18 \citep{10.1007/978-3-319-46493-0_38} architecture. For Abram, we set $\gamma=1, h=\varepsilon, \varepsilon=16/255$. Similar as in MNIST experiments, practically we find that setting larger noise scale for attack in training Abram helps to obtain a better final evaluation performance. The batch size $N=128$ and the inner loop is simulated for $T=10$ times. We train both mini-batching Abram and FGSM for 30 epochs. Due to its worse performance for MNIST and the large size of CIFAR10, we have not used the non-mini-batching version of Abram in this second problem. For the Bayesian sample attack and the Bayesian mean attack, we set $\gamma=0.001$. We present the results in Table~\ref{Table:deep_learning_test_acc_cifar10}. There, we observe that mini-batching Abram outperforms FGSM under Wasserstein and the Bayesian attacks, but not in any of the other cases.

\begin{table}[h!]
    \centering 
	\begin{tabular}{rlcllc}
		\toprule
		&\textbf{Adversarial Attack ($\varepsilon=8/255$)} &  \textbf{Mini-batching Abram}  & \quad\quad\textbf{FGSM} &\\ \midrule
		&Benign Test &   65.35$\pm$0.05   & \quad\quad 55.61$\pm$0.03  &     \\
            \hdashline
		&Auto-PGD 
  &    11.15$\pm$0.12      & \quad\quad 43.70$\pm$0.06  &   \\ 
            &PGD 
            &     11.22$\pm$0.09      & \quad\quad 43.65$\pm$0.04   &        \\ 
		&Wasserstein Attack 
  &    58.04$\pm$0.15     &\quad\quad  55.30$\pm$0.03    &      \\ 
		
		&Carlini \& Wagner Attack 
  &    19.01$\pm$0.12   &  \quad\quad 62.60$\pm$0.02& \\ 
		&Bayesian sample attack   &     62.52$\pm$0.03       &  \quad\quad 55.83$\pm$0.05    &    \\ 
		&Bayesian mean attack &     63.72$\pm$0.06 &   \quad\quad 55.81$\pm$0.05  & \\
		\bottomrule
	\end{tabular}
	    \caption{\small Comparison of test accuracy (\%) on CIFAR10 with different adversarial attack after mini-batching Abram and FGSM \citep{Wong2020Fast} adversarial training.}
	\label{Table:deep_learning_test_acc_cifar10}
\end{table}

\section{Conclusions} \label{Sec_concl}
We have introduced the Bayesian adversarial robustness problem. This problem can be interpreted as either a relaxation of the usual minmax problem in adversarial learning or as learning methodology that is able to counter Bayesian  adversarial attacks. To solve the Bayesian adversarial robustness problem, we introduce Abram -- the Adversarially Bayesian Particle Sampler. We prove that Abram approximates a McKean-Vlasov SDE and that this McKean-Vlasov SDE is able to find the minimiser of certain (simple) Bayesian adversarial robustness problems. Thus, at least for a certain class of problems, we give a mathematical justification for the use of Abram. We propose two ways to discretise Abram: a direct Euler-Maruyama discretisation of the Abram dynamics and an alternative method that is more suitable when training with respect to large data sets. We apply Abram in two deep learning problems. There we see that Abram can effectively prevent certain adversarial attacks (especially Bayesian attacks), but is overall not as strong as classical optimisation-based heuristics.

\bibliography{ref}{}
\bibliographystyle{abbrvnat}

\vspace{0.5cm}

\noindent \textbf{Competing interests:} The authors declare none.
\end{document}